\documentclass{article} 
\usepackage{iclr2022_conference,times}

\usepackage{amsmath,amsfonts,bm}

\def\1{\bm{1}}

\def\rt{{T}}
\def\ru{{U}}

\def\ry{{Y}}

\def\rve{{\mathbf{e}}}

\def\rvt{{T}}

\def\rvv{{V}}
\def\rvw{{W}}
\def\rvx{{X}}
\def\rvy{{Y}}
\def\rvz{{Z}}

\def\vtheta{{\bm{\theta}}}
\def\va{{\bm{a}}}
\def\vb{{\bm{b}}}
\def\vc{{\bm{c}}}
\def\vd{{\bm{d}}}

\def\vf{{\bm{f}}}
\def\vg{{\bm{g}}}
\def\vh{{\bm{h}}}
\def\vi{{\bm{i}}}
\def\vj{{\bm{j}}}
\def\vk{{\bm{k}}}

\def\vr{{\bm{r}}}
\def\vs{{\bm{s}}}
\def\vt{{\bm{t}}}
\def\vu{{\bm{u}}}
\def\vv{{\bm{v}}}

\def\vx{{\bm{x}}}
\def\vy{{\bm{y}}}
\def\vz{{\bm{z}}}

\def\mA{{\bm{A}}}

\def\mC{{\bm{C}}}

\def\mG{{\bm{G}}}

\def\mI{{\bm{I}}}
\def\mJ{{\bm{J}}}

\def\mL{{\bm{L}}}

\def\mO{{\bm{O}}}

\def\mR{{\bm{R}}}

\def\mX{{\bm{X}}}

\DeclareMathAlphabet{\mathsfit}{\encodingdefault}{\sfdefault}{m}{sl}
\SetMathAlphabet{\mathsfit}{bold}{\encodingdefault}{\sfdefault}{bx}{n}

\newcommand{\E}{\mathbb{E}}

\newcommand{\R}{\mathbb{R}}

\newcommand{\KL}{D_{\mathrm{KL}}}
\newcommand{\Var}{\mathrm{Var}}

\usepackage{hyperref}
\usepackage{url}

\usepackage{amsmath,amssymb,amsfonts,amsthm,bm}
\newtheorem{theorem}{Theorem}
\newtheorem{proposition}{Proposition}
\newtheorem{corollary}{Corollary}
\theoremstyle{definition}
\newtheorem{definition}{Definition}
\usepackage{mathtools}
\usepackage{graphicx}
\usepackage{paralist}
\usepackage[linesnumbered,ruled,noend]{algorithm2e}
\usepackage{xcolor}

\SetCommentSty{mycommfont}
\usepackage{subcaption}
\usepackage{tabularx}
\usepackage{booktabs}
\usepackage{caption}
\usepackage{wrapfig}
\usepackage{centernot}
\usepackage{enumitem}   

\newcommand{\independent}{\raisebox{0.05em}{\rotatebox[origin=c]{90}{$\models$}}}

\DeclareMathOperator{\diag}{diag}

\DeclareMathOperator{\bern}{Bern}
\DeclareMathOperator{\logi}{Logi}
\newcommand{\blambda}{\bm\lambda}
\newcommand{\biglambda}{\bm\Lambda}
\newcommand{\btheta}{\bm\theta}
\newcommand{\bphi}{\bm\phi}
\newcommand{\bgamma}{\bm\gamma}
\newcommand{\bbeta}{\bm\beta}
\newcommand{\beps}{\bm\epsilon}
\newcommand{\bsig}{\bm \sigma}
\newcommand{\inv}{^{-1}}
\newcommand{\st}{^*}
\newcommand{\PS}{\mathbb{P}}

\newcommand{\M}{\mathbb{M}}
\newcommand{\data}{\mathcal{D}}
\newcommand{\x}{\rvx}
\newcommand{\y}{\rvy}
\newcommand{\z}{\rvz}

\newcommand{\condpriorparam}{p_{\bm\lambda}(\vz|\vx,t)}
\newcommand{\decoderparam}{p_{\vf}(\vy|\vz,t)}
\newcommand{\encoderparam}{q_{\bphi}(\vz|\vx,\vy,t)}

\newcommand{\encoder}{q(\vz|\vx,\vy,t)}
\newcommand{\vaegenparam}{p_{\vtheta}(\vy,\vz|\vx,t)}
\newcommand{\vaegen}{p(\vy,\vz|\vx,t)}
\newcommand{\vaeobsparam}{p_{\vtheta}(\vy|\vx,t)}

\newcommand{\vaepostparam}{p_{\vtheta}(\vz|\vx,\vy,t)}

\newcommand{\trueobs}{p(\vy|\vx,t)}
\newcommand{\truejoint}{p(\vx,\vy,t)}

\newcommand{\e}{{\rve}}
\newcommand{\binset}{\{0,1\}}

\setlist{nosep} 
\AtBeginDocument{%
\setlength{\belowdisplayskip}{4pt plus .0pt minus 2.0pt} \setlength{\belowdisplayshortskip}{2pt plus .0pt minus 2.0pt}
\setlength{\abovedisplayskip}{4pt plus .0pt minus 2.0pt} \setlength{\abovedisplayshortskip}{2pt plus .0pt minus 2.0pt}
}
\makeatletter
\def\thm@space@setup{%
  \thm@preskip=3.0pt plus 1.0pt minus 2.0pt
  \thm@postskip=.0pt plus .0pt minus .0pt
}
\makeatother

\theoremstyle{plain}
\newtheorem{lemma}{Lemma}

\title{
\vspace{-.3in}
$\beta$-Intact-VAE: Identifying and Estimating \\ Causal Effects under Limited Overlap
\vspace{-.1in}
}

\author{
Pengzhou (Abel) Wu \& Kenji Fukumizu \\
Department of Statistical Science, The Graduate University for Advanced Studies \\
\& The Institute of Statistical Mathematics \\
Tachikawa, Tokyo \\
\texttt{\{wu.pengzhou,fukumizu\}@ism.ac.jp} \\
\vspace{-.3in}
}

\iclrfinalcopy 
\begin{document}

\maketitle

\begin{abstract}
As an important problem in causal inference, we discuss the identification and estimation of treatment effects (TEs) under limited overlap; that is, when subjects with certain features belong to a single treatment group. We use a latent variable to model a prognostic score which is widely used in biostatistics and sufficient for TEs; i.e., we build a generative prognostic model. We prove that the latent variable recovers a prognostic score, and the model identifies individualized treatment effects. The model is then learned as $\beta$-Intact-VAE––a new type of variational autoencoder (VAE). We derive the TE error bounds that enable representations balanced for treatment groups conditioned on individualized features. The proposed method is compared with recent methods using (semi-)synthetic datasets. 
\end{abstract}

\section{Introduction}

Causal inference \citep{imbens2015causal, pearl2009causality}, i.e, inferring causal effects of interventions, is a fundamental field of research. In this work, we focus on treatment effects (TEs) based on a set of observations comprising binary labels $\rt$ for treatment/control (non-treated), outcome $\y$, and other covariates $\x$.  Typical examples include estimating the effects of public policies or new drugs based on the personal records of the subjects. The fundamental difficulty of causal inference is that we never observe \textit{counterfactual} outcomes that would have been if we had made the other decision (treatment or control). While randomized controlled trials (RCTs) control biases through randomization and are ideal protocols for causal inference, they often have ethical and practical issues, or suffer from expensive costs. Thus, causal inference from observational data is important. 

Causal inference from observational data has other challenges as well. One is \textit{confounding}: there may be variables, called confounders, that causally affect both the treatment and the outcome, and spurious correlation/bias follows. The other is the systematic \textit{imbalance} (difference) of the distributions of the covariates between the treatment and control groups––that is, $\x$ depends on $\rt$, which introduces bias in estimation. A majority of studies on causal inference, including the current work, have relied on unconfoundedness; this means that appropriate covariates are collected so that the confounding can be controlled by conditioning on the covariates. 
However, such high-dimensional covariates tend to introduce a stronger imbalance between treatment and control. 

The current work studies the issue of imbalance in estimating individualized TEs conditioned on $\x$. Classical approaches aim for 
\textit{covariate balance}, $\x$ independent of $\rt$, 
by matching and re-weighting \citep{Stuart2010matching,rosenbaum2020modern}. Machine learning methods have also been exploited; there are semi-parametric methods––e.g.,  \citet[TMLE]{van2018targeted}––which improve finite sample performance, as well as non-parametric methods––e.g., \citet[CF]{wager2018estimation}. Notably, from \citet{johansson2016learning}, there has been a recent increase in interest in \textit{balanced representation learning} (BRL) to learn representations $\z$ of the covariates, such that $\z$ independent of $\rt$. 

The most serious form of imbalance is the \textit{limited (or weak) overlap of covariates}, which means that sample points with certain covariate values belong to a single treatment group. In this case, a straightforward estimation of TEs is not possible at non-overlapping covariate values due to lack of data. Some works have focused on providing robustness to limited overlap \citep{armstrong2021finite},  trimming non-overlapping sample points \citep{10.1093/biomet/asy008}, or studying convergence rates based on overlap \citep{hong2020inference}. Limited overlap is particularly relevant to machine learning methods that exploit high-dimensional covariates. This is because, with higher-dimensional covariates, overlap is harder to satisfy and verify \citep{d2020overlap}.

To address imbalance and limited overlap, we use a prognostic score \citep{hansen2008prognostic}; it is a sufficient statistic of outcome predictors and is among the key concepts of sufficient scores for TE estimation. As a function of covariates, 
it can map some non-overlapping values to an overlapping value in a space of lower-dimensions.
For individualized TEs, we consider \textit{conditionally balanced representation} $\z$, such that $\z$ is independent of $\rt$ given $\x$––which, as we will see, is a necessary condition for a balanced prognostic score. 
Moreover, prognostic score modeling can benefit from methods in predictive analytics and exploit rich literature, particularly in medicine and health \citep{hajage2017estimation}. Thus, it is promising to combine the predictive power of prognostic modeling and machine learning. With this idea, our method builds on a generative prognostic model that models the prognostic score as a latent variable and factorizes to the score distribution and outcome distribution.

As we consider latent variables and causal inference, \textit{identification} is an issue that must be discussed before estimation is considered. ``Identification'' means that the parameters of interest (in our case, representation function and TEs) are uniquely determined and expressed using the true observational distribution.
Without identification, a consistent estimator is impossible to obtain, and a model would fail silently; in other words, the model may fit perfectly but will return an estimator that converges to a wrong one, or does not converge at all \citep[particularly Sec.~8]{lewbel2019identification}. Identification is even more important for causal inference; because, unlike usual (non-causal) model misspecification, causal assumptions are often unverifiable through observables \citep{white2013identification}. Thus, it is critical to specify the theoretical conditions for identification, and then the applicability of the methods can be judged by knowledge of an application domain.

A major strength of our generative model is that the latent variable is identifiable. This is because the factorization of our model is naturally realized as a combination of identifiable VAE \citep[iVAE]{khemakhem2020variational} and conditional VAE \citep[CVAE]{sohn2015learning}. Based on model identifiability, we develop two identification results for individualized TEs under limited overlap.
A similar VAE architecture was proposed in \citet{wu2021identifying,wu2021intact}; the current study is different in setting, theory, learning objective, and experiments. The previous work studies unobserved confounding but not limited overlap, with different set of assumptions and identification theories.
The current study further provides bounds on individualized TE error, and the bounds justify a conditionally balancing term controlled by hyperparameter $\beta$, as an interpolation between the two identifications.

In summary, we study the identification (Sec.~\ref{sec:model&id}) and estimation (Sec.~\ref{sec:estimation}) of individualized TEs under limited overlap. 
Our approach is based on recovering prognostic scores from observed variables. To this end, our method exploits recent advances in identifiable representation––particularly iVAE.
The code is in Supplementary Material, and the proofs are in Sec.~\ref{sec:proofs}. Our main contributions are: 
\vspace{-5pt}
\begin{enumerate}[topsep=0pt, partopsep=0pt, itemsep=0pt, parsep=0pt, leftmargin=13pt]
\item[1)] TE identification under limited overlap of $\x$, via prognostic scores and an identifiable model;
\item[2)] bounds on individualized TE error, which justify our conditional BRL;
\item[3)] 
a new regularized VAE, $\beta$-Intact-VAE,
realizing the identification and conditional balance;
\item[4)] experimental comparison to the state-of-the-art methods on (semi-)synthetic datasets.
\end{enumerate}

\subsection{Related work}

\textbf{Limited overlap.}~Under limited overlap, \citet{luo2017estimating} estimate the average TE (ATE) by reducing covariates to a linear prognostic score. \citet{farrell2015robust} estimates a constant TE under a partial linear outcome model. \citet{d2021deconfounding} study the identification of ATE by a general class of scores, given the (linear) propensity score and prognostic score.
Machine learning studies on this topic have focused on finding overlapping regions \citep{oberst2020characterization,dai2020quantifying}, or indicating possible failure under limited overlap \citep{jesson2020identifying}, but not remedies. An exception is \citet{johansson2020generalization}, which provides bounds under limited overlap.
To the best of our knowledge, our method is the first machine learning method that provides identification under limited overlap.

\textbf{Prognostic scores\space} have been recently combined with machine learning approaches, mainly in the biostatistics community. 
For example, 
\citet{huang2017joint} estimate individualized TE by reducing covariates to a linear score which is a joint propensity-prognostic score.  \citet{tarr2021estimating} use SVM to minimize the worst-case bias due to prognostic score imbalance. 
However, in the machine learning community, few methods consider prognostic scores; \citet{zhang2020treatment} and \citet{hassanpour2019learning} learn outcome predictors, without mentioning prognostic score––while \citet{johansson2020generalization} conceptually, but not formally, connects BRL to prognostic score. Our work is the first to formally connect generative learning and prognostic scores for TE estimation.

\textbf{Identifiable representation.}~Recently, independent component analysis (ICA) and representation learning––both ill-posed inverse problems––meet together to yield nonlinear ICA and identifiable representation; for example, using
VAEs \citep{khemakhem2020variational}, and energy models \citep{khemakhem2020ice}. 
The results are exploited in causal discovery \citep{pmlr-v108-wu20b} and out-of-distribution generalization \citep{sun2020latent}. This study is the first to explore identifiable representations in TE identification.  

\textbf{BRL and related methods\space} amount to a major direction. Early BRL methods include BLR/BNN \citep{johansson2016learning} and TARnet/CFR \citep{shalit2017estimating}. In addition, \cite{yao2018representation} exploit the local similarity between data points. \cite{shi2019adapting} use similar architecture to TARnet, considering the importance of treatment probability. There are also methods that use GAN \citep[GANITE]{yoon2018ganite} and Gaussian processes \citep{alaa2017bayesian}. Our method shares the idea of BRL, and further extends to conditional balance––which is natural for individualized TE. 

\textbf{Others.}~Our work can lay conceptual and theoretical foundations of VAE methods for TEs  \citep[e.g.,][]{louizos2017causal,lu2020reconsidering}. In addition, some studies consider monotonicity, which is injectivity on $\R$, together with overlap \citep{johansson2020generalization,zhang2020learning}. See Sec.~\ref{sec:other_related}.

\begingroup

\section{Setup and preliminaries}
\label{sec:setup}

\subsection{Counterfactuals, treatment effects, and identification}
Following \citet{imbens2015causal}, we assume there exist \textit{potential outcomes} $\rvy(t) \in \R^d, t \in \{0,1\}$. $\rvy(t)$ is the outcome that would have been observed if the treatment value $\rt=t$ was applied. We see $\rvy(t)$ as the hidden variables that give the \textit{factual outcome} $\y$ under \textit{factual assignment} $\rt=t$. Formally, $\rvy(t)$ is defined by the \textit{consistency of counterfactuals}: $\rvy = \rvy(t)$ if $\rt=t$; or simply $\y=\y(\rt)$. The \textit{fundamental problem of causal inference} is that, for a unit under research, we can observe only one of $\y(0)$ or $\y(1)$––w.r.t.~the treatment value applied. 
That is, ``factual'' refers to $\y$ or $\rt$, which is \textit{observable}; or estimators built on the observables. We also observe relevant covariate(s) $\rvx \in \mathcal{X} \subseteq \R^m$, which is associated with individuals, with distribution $\data \coloneqq (\x,\y,\rt) \sim \truejoint$.
We use upper-case (e.g. $\rt$) to denote random variables, and lower-case (e.g. $t$) for realizations.

The expected potential outcome is denoted by $\mu_t(\vx) = \E(\rvy(t)|\rvx=\vx)$ conditioned on $\rvx=\vx$.
The estimands in this work are the conditional ATE (CATE) and ATE, defined, respectively, by: 
\begin{equation}
\label{eq:ce}
    \tau(\vx) = \mu_1(\vx) - \mu_0(\vx),\quad \nu = \E(\tau(\x)).
\end{equation}
CATE is seen as an \textit{individual-level}, personalized, treatment effect, given
highly discriminative $\x$.

Standard results \citep{rubin2005causal}\citep[Ch.~3]{hernanCausalInferenceWhat2020} show sufficient conditions for TE identification in general settings. They are \emph{ Exchangeability}: $\rvy(t) \independent \rt| \rvx$, and \emph{Overlap}: $p(t|\vx) > 0$ for any $\vx\in\mathcal{X}$. Both are required for $t\in\binset$. When $t$ appears in statements without quantification, we always mean ``for both $t$''. Often, \textit{Consistency} is also listed; however, as mentioned, it is better known as the well-definedness of counterfactuals.
Exchangeability means, just as in RCTs, but additionally given $\x$, that  there is no correlation between factual $\rt$ and potential $\ry(t)$. Note that the popular assumption $\rvy(0),\rvy(1) \independent \rt| \rvx$ is stronger than $\rvy(t) \independent \rt| \rvx$ and is not necessary for identification \citep[pp.~15]{hernanCausalInferenceWhat2020}.
Overlap means that the supports of $p(\vx|t=0)$ and $p(\vx|t=1)$ should be the same, and this ensures that there are data for $\mu_t(\vx)$ on any $(\vx, t)$.

We rely on consistency and exchangeability, but in Sec.~\ref{sec:identification}, will relax the condition of the overlapping covariate to allow some non-overlapping values $\vx$––that is, covariate $\x$ is \textit{limited-overlapping}.
In this paper, we also discuss overlapping variables other than $\x$ (e.g., prognostic scores), and provide a definition for any random variable $\rvv$ with support $\mathcal{V}$ as follows: 

\begin{definition}
$\rvv$ is \textit{Overlapping} if $p(t|V=\vv) > 0$ for any $t\in\binset,\vv \in \mathcal{V}$. If the condition is violated at some value $\vv$, then $\vv$ is \textit{non-overlapping} and $\rvv$ is \textit{limited-overlapping}. 
\end{definition}

\subsection{Prognostic scores}
\label{sec:pgs}

Our method aims to recover a prognostic score \citep{hansen2008prognostic}, adapted as a Pt-score in Definition \ref{def:scores}. On the other hand, balancing scores \citep{rosenbaum1983central} $\vb(\x)$ are defined by $\rt\independent\x| \vb(\x)$, of which the propensity score $p(t=1|\x)$ is a special case.  See Sec.~\ref{sec:scores} for detail.

\begin{definition}
\label{def:scores}
A \textit{Pt-score} (PtS) is two functions $\PS_t(\x)$ ($t=0,1$) such that $\rvy(t)\independent\x|\PS_t(\x)$.
A PtS is called a \textit{P-score} (PS) if $\PS_0=\PS_1$.
\end{definition}
Note that a PtS is by definition two functions; thus, overlapping $\PS_t(\x)$ means \textit{both} $\PS_0(\x)$ and $\PS_1(\x)$ are overlapping. 
\textbf{Why not balancing scores?} 
While balancing scores $\vb(\x)$ have been widely used in causal inference, PtSs are more suitable for discussing overlap.
Our purpose is to recover an overlapping score for limited-overlapping $\x$. It is known that overlapping $\vb(\x)$ implies overlapping $\x$ \citep{d2020overlap}, which counters our purpose.
In contrast, overlapping PS does not imply overlapping $\vb(\x)$ (see Sec.~\ref{sec:pgs_vs_bs} for an example).
Moreover, with theoretical and experimental evidence, it is recently conjectured that PtSs maximize overlap among a class of sufficient scores, including $\vb(\x)$ \citep{d2021deconfounding}. In general, \cite{hajage2017estimation} show that prognostic score methods perform better––or as well as––propensity score methods.

Below is a corollary of Proposition 5 in \cite{hansen2008prognostic}; note that $\PS_t(\x)$ satisfies exchangeability.
\begin{proposition}[Identification via PtS]
\label{cate_by_bts}
If $\PS_t$ is a PtS and $\y|\PS_{\hat t}(\x),T \sim p_{\y|\PS_{\hat t},T}(\vy|P,t)$ where $\hat{t}\in \binset$ is a counterfactual assignment, then CATE and ATE are identified, using \eqref{eq:ce} and
\begin{equation}
\label{eq:cate_by_bts}
    \begin{split}
        \mu_{\hat{t}}(\vx) &= \E(\rvy({\hat{t}})|\PS_{\hat{t}}(\x),\x=\vx) = \E(\rvy|\PS_{\hat{t}}(\vx),\rt=\hat{t})=\textstyle \int p_{\y|\PS_{\hat t},T}(\vy|\PS_{\hat{t}}(\vx),\hat{t})\vy d\vy
    \end{split}
\end{equation}
\end{proposition}
\vspace{-.1in}
With the knowledge of $\PS_t$ and $p_{\y|\PS_{\hat t},T}$, we choose one of $\PS_0,\PS_1$ and set $t=\hat{t}$ in the density function, w.r.t~the $\mu_{\hat t}$ of interest. This counterfactual assignment resolves the problem of non-overlap at $\vx$.
Note that a sample point with $\x=\vx$ may not have $\rt=\hat{t}$.

We mainly consider additive noise models for $\y(t)$,  which ensures the existence of PtSs.
\begin{enumerate}
\renewcommand{\labelenumi}{\textbf{(G1)}}
\def\theenumi{\textbf{(G1)}}
\item\hspace{-2mm}\footnote{The symbols \textbf{G}, \textbf{M}, and \textbf{D} in the labels of conditions stand for Generating process, Model, and Data.} \label{ass:anm} (Additive noise model) the data generating process (DGP) for $\y$ is 
$\y = \vf\st(\M(\x),{\rt}) + \e$
where $\vf\st,\M$ are functions, and $\e$ denotes a zero-mean exogenous (external) noise.
\end{enumerate}
The potential outcomes can be defined by the DGP as $\y(t):=\vf\st(\M(\x),t) + \e$. The DGP also specifies how other variables causally affect $\y$. For example, $\x$ affects $\y$ through $\M$; and thus $\M(\x)$ is the effect modifier \citep{hansen2008prognostic}––which is often components of $\x$ affecting $\y$ directly. 
Additive noise models are used in nonparametric regression methods for TEs \citep{caron2020estimating}.

Under \ref{ass:anm}, we can find natural 
examples of PS and PtS. 
1) $\PS_t(\x)\coloneqq\vf\st_t(\M(\x))=\mu_t(\x)$\footnote{We often write $t$ of the function argument in subscripts, indicating possible counterfactual assignments.} is a PtS
and not PS; 2) $\M$ is a PS ($\x$ is a trivial PS); and 3) $\PS(\x)\coloneqq(\mu_0(\x), \mu_1(\x))$ is a PS.

We use PS and PtS to construct a representation for CATE estimation, and their balance is important.   
Obviously, a PS $\PS(\x)$ is a conditionally balanced representation (defined as $\z\independent\rt|\x$ in Introduction), because $\PS(X)$ does not depend on $\rt$ given $\x$. Note that a PtS is conditionally balanced if and only if it is a PS. Thus, we introduce the notion of \textit{balanced PtS} in a non-rigorous way: a PtS $\PS_t$ is called balanced if the value of a measure for the conditional independence $\PS_{\rt}(\x) \independent \rt | \x$ is small.

\section{Identification under generative prognostic model}
\label{sec:model&id}

\begingroup

\setlength{\columnsep}{8pt}
\setlength{\intextsep}{5pt}
\begin{wrapfigure}{r}{0.26\columnwidth}
  \vspace{-.2in}
  \begin{center}
    \includegraphics[width=0.26\columnwidth]{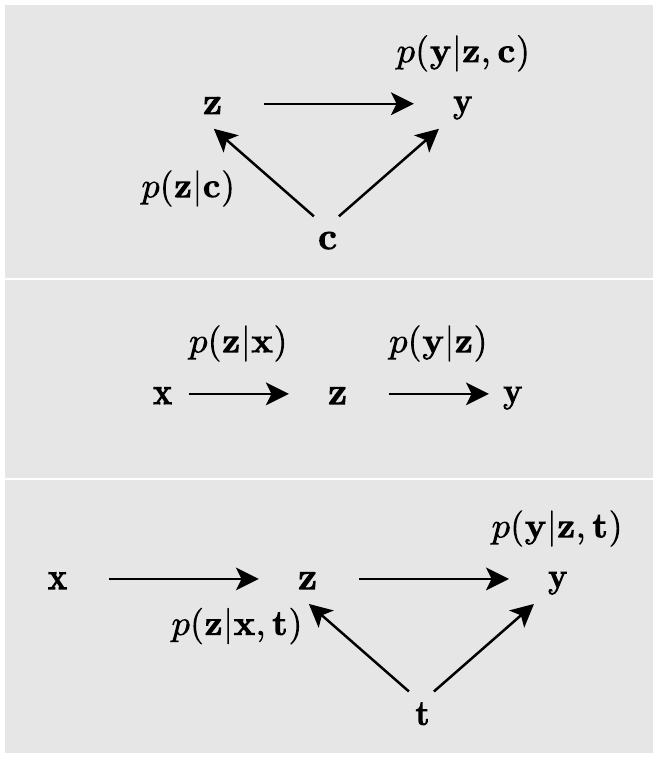}
  \end{center}
  \vspace{-.15in}
  \caption{CVAE, iVAE, and Intact-VAE: Graphical models of the decoders. 
  }
\vspace{-.5in}
\label{f:vae}
\end{wrapfigure}

In Sec.~\ref{sec:arch}, we specify the generative prognostic model $\vaegen$, and show its identifiability. In Sec.~\ref{sec:identification}, we prove the identification of CATEs, which is one of our main contributions. The theoretical analysis involves only our generative model (i.e., prior and decoder), but not the encoder.
The encoder is not part of the generative model and is involved as an approximate posterior in the estimation, which is studied in Sec.~\ref{sec:estimation}.

\subsection{Model, architecture, and identifiability}
\label{sec:arch}
We present the necessary definitions and results, and defer more explanations to Sec.~\ref{sec:arch_details}. 
The contents of this subsection are essentially taken from \cite{wu2021principled}, but included here for completeness.

Our goal is to build a model that can be learned by VAE from observational data to obtain a PtS, or better, a PS, via the latent variable $\z$. The generative prognostic model of the proposed method is as follows: 
\begin{equation}
\label{model_indep}
    \begin{gathered}
        p_{\vtheta}(\vy,\vz|\vx,t) = p_{\vf}(\vy|\vz,t)p_{\blambda}(\vz|\vx,t), \\
        p_{\vf}(\vy|\vz,t)=p_{\beps}(\vy-\vf_{t}(\vz)),
         \quad 
        p_{\blambda}(\vz|\vx,t) \sim \mathcal{N} ( \vz ; \vh_{t}(\vx),\diag(\vk_{t}(\vx))).
    \end{gathered}
\end{equation}
The factor $p_{\vf}(\vy|\vz,t)$ is our decoder, which models $p_{\y|\PS_t,T}(\vy|P,t)$ in \eqref{eq:cate_by_bts}; and $p_{\blambda}(\vz|\vx,t)$ is the conditional prior, which models $\PS_{\rt}(\x)$. 
The outcome assumes an additive noise model such that 
$\beps \sim p_{\beps}$ denotes the noise model. The prior is a factorized Gaussian, where $\blambda_{\rt}(\x)\coloneqq \diag\inv(\vk_{\rt}(\rvx))(\vh_{\rt}(\x), -\frac{1}{2})^T$ is the natural parameter as in the exponential family. $\vtheta\coloneqq(\vf,\blambda)=(\vf, \vh, \vk)$ contains the functional parameters. We denote $n\coloneqq\dim(\z)$.

For inference, the standard argument derives the ELBO: 
\begin{equation}
\label{elbo}
    \log \trueobs 
    \geq \E_{\vz \sim q}\log \decoderparam - \KL(\encoder\Vert \condpriorparam) .
\end{equation}
Note that the encoder $q$ conditions on all the observables $(\rvx,\rvy,\rt)$; this fact plays an important role in Section \ref{sec:regular}. This architecture is called \textit{Intact-VAE} (\textit{I}de\textit{n}tifiable \textit{t}re\textit{a}tment-\textit{c}ondi\textit{t}ional VAE). 
See Figure \ref{f:vae} for comparison in terms of graphical models. See Sec.~\ref{vaes} for basics of VAEs.

\endgroup

Our model identifiability extends the theory of iVAE,
and the following conditions are inherited. 
\begin{enumerate}
\renewcommand{\labelenumi}{\textbf{(M1)}}
\def\theenumi{\textbf{(M1)}}
\item \label{ass:model} i) $\vf_t$ is injective, and ii) $\vf_t$ is differentiable. 
\end{enumerate}
\begin{enumerate}
\renewcommand{\labelenumi}{\textbf{(D1)}}
\def\theenumi{\textbf{(D1)}}
\item \label{ass:inv_jac_smp} $\blambda_t(\x)$ is non-degenerate, i.e., the linear hull of its support is $2n$-dimensional. 
\end{enumerate}
Under \textbf{(M1)} and \textbf{(D1)}, we obtain the following identifiability of the parameters in the model: if $\vaeobsparam=p_{\vtheta'}(\vy|\vx,t)$, we have, for any $\vy_t$ in the image of $\vf_t$:
\begin{equation}
\label{eq:class}
     \vf_t^{-1}(\vy_t) = \diag(\va){\vf'_t}\inv(\vy_t) + \vb 
     =:\mathcal{A}_t({\vf'_t}\inv(\vy_t))
\end{equation}
where $\diag(\va)$ is an invertible $n$-diagonal matrix and $\vb$ is an $n$-vector, both of which depend on $\bm\lambda_t(\vx)$ and $\bm\lambda'_t(\vx)$. 
The essence of the result is that $\vf'_t = \vf_t \circ \mathcal{A}_t$; that is, $\vf_t$ can be identified (learned) up to an affine transformation $\mathcal{A}_t$. See Sec.~\ref{sec:proofs} for the proof and a relaxation of \ref{ass:inv_jac_smp}.
In this paper, symbol $'$ (prime) always indicates another parameter (variable, etc.): $\vtheta'=(\vf',\blambda')$.

\subsection{Identifications under limited-overlapping covariate}
\label{sec:identification}

In this subsection, we present two results of CATE identification based on the recovery of equivalent PS and PtS, respectively.  
Since PtSs are functions of $\x$, the theory assumes a noiseless prior for simplification, i.e., $\vk(\x)=\bm0$; the prior $\z_{\blambda,t} \sim p_{\blambda}(\vz|\vx,t)$ degenerates to function $\vh_{t}(\vx)$. 

PtSs with dimensionality lower than or equal to $d=\dim(\y)$ are essential to address limited overlapping, as shown below.
We set $n=d$ because $\mu_t$ is a PtS of the same dimension as $\y$ under \ref{ass:anm}. In practice, $n=d$ means that we seek a low-dimensional representation of $\x$.  
\ref{ass:pt_inj} makes the dimensionality explicit and reduces to \ref{ass:anm} if the only possibility is that $\PS_t=\mu_t$ and $\vj_t$ is identity. 
\begin{enumerate}
\renewcommand{\labelenumi}{\textbf{(G2)}}
\def\theenumi{\textbf{(G2)}}
\item \label{ass:pt_inj} (Low-dimensional PtS) Under \ref{ass:anm}, $\mu_{t}(\x)=\vj_{t}(\PS_t(\x))$
for some $\PS_t$ and injective $\vj_t$.
\end{enumerate}
We use \ref{ass:pt_inj} instead of \ref{ass:anm} hereafter. Clearly, $\PS_t$ in \ref{ass:pt_inj} is PtS. In addition, injectivity and $n=d$ ensure that $n=\dim(\y) \geq \dim(\PS_t)$. Similarly, \ref{ass:p_inj} ensures that $n \geq \dim(\PS)$ for PS. 

\begin{enumerate}
\renewcommand{\labelenumi}{\textbf{(G2')}}
\def\theenumi{\textbf{(G2')}}
\item \label{ass:p_inj} (Low-dimensional PS) Under \ref{ass:anm}, $\mu_{t}(\x)=\vj_{t}(\PS(\x))$
for some $\PS$ and injective $\vj_t$.
\end{enumerate}
\ref{ass:p_inj} means that CATEs
are given by $\mu_0$ and an invertible function $\vi\coloneqq\vj_1\circ\vj_0\inv$. See Sec.~\ref{sec:p_inj} for more discussions and real-world examples. 

With \ref{ass:pt_inj} or \ref{ass:p_inj}, overlapping $\x$ can be relaxed to overlapping PS or PtS plus the following:

\begin{enumerate}
\renewcommand{\labelenumi}{\textbf{{(M2)}}} 
\def\theenumi{\textbf{(M2)}} 
\item \label{ass:prepart} \text{(Score partition preserving) For any $\vx,\vx' \in \mathcal{X}$, if $\PS_t(\vx)=\PS_t(\vx')$, then $\vh_t(\vx)=\vh_t(\vx')$.}
\end{enumerate}
Note that \ref{ass:prepart} is only required for optimal $\vh$ that satisfies $\E_{p_{\btheta}}(\y|\x,\rt)=\E(\y|\x,\rt)$ in Proposition \ref{th:id_nop_pscore}, or $\vaeobsparam=\trueobs$ in Theorem \ref{th:id_nop_ptscore}. The intuition is that $\PS_t$ maps non-overlapping $\vx$ to an overlapping value, and $\vh_t$ preserves this property through learning. 
Linear $\PS_t$ and $\vh_t$ imply \ref{ass:prepart} and are often assumed, e.g., in \cite{huang2017joint,luo2017estimating,d2021deconfounding}. Linear outcome models \citep{farrell2015robust,schuler2020increasing} are also common.

Our first identification, Proposition \ref{th:id_nop_pscore}, relies on \ref{ass:p_inj} and our generative model, \textit{without} model identifiability (so differentiable $\vf_t$ is not needed). 

\begin{proposition}[Identification via recovery of PS]
\label{th:id_nop_pscore}
Suppose we have DGP \ref{ass:p_inj} and model \eqref{model_indep} with $n=d$. Assume \ref{ass:model}-i) and\renewcommand{\labelenumi}{\textbf{\emph{(M3)}}}\begin{inparaenum}\def\theenumi{\textbf{(M3)}}
\item \label{ass:match_pscore} \emph{(PS matching) $\vh_0(\x)=\vh_1(\x)$ and $\vk(\x) = \bm0$.}
\end{inparaenum}
Then, if $\E_{p_{\btheta}}(\y|\x,\rt)=\E(\y|\x,\rt)$, we have
\renewcommand{\labelenumi}{\arabic{enumi})}
\begin{enumerate}
\def\theenumi{\arabic{enumi})}
\item (Recovery of PS) $\vz_{\blambda,t}=\vh_t(\vx)=\vv(\PS(\vx))$
on overlapping $\vx$, \\ where $\vv: \mathcal{P} \to \R^n$ is an injective function, and $\mathcal{P}\coloneqq\{\PS(\vx)|\text{overlapping }\vx\}$.
\item (CATE Identification) if $\PS$ in \ref{ass:p_inj} is overlapping, and \ref{ass:prepart} is satisfied, then 
$\mu_{{t}}(\vx) = \hat{\mu}_{{t}}(\vx)$ for any ${t} \in \binset$ and $\vx \in \mathcal{X}$, where $\hat{\mu}_{{t}}(\vx) \coloneqq \E_{p_{\blambda}(\z| \vx,t)} \E_{p_{\vf}}(\y|\z,{t}) = \vf_t(\vh_t(\vx))$.
\end{enumerate}
\end{proposition}
In essence, i) the true DGP is identified up to an invertible mapping $\vv$, such that $\vf_t=\vj_t\circ\vv\inv$ and $\vh_t=\vv\circ\PS_t$; and ii) $\PS_t$ is recovered up to $\vv$, and $\y(t)\independent\x|\PS_t(\x)$ is preserved––with \textit{same} $\vv$ for both $t$. Theorem \ref{th:id_nop_ptscore} below also achieves the essence i) and ii), under $\PS_0\neq\PS_1$.

The existence of PS is more preferred, 
because it satisfies overlap and \ref{ass:prepart} more easily than PtS which 
requires the conditions for each of the two functions of PtS.
However, the existence of low-dimensional PS is uncertain in practice when our knowledge of the DGP is limited. Thus, we depend on Theorem \ref{th:id_nop_ptscore} based on the model identifiability to work under PtS which 
generally exists. 
\begin{theorem}[Identification via recovery of PtS]
\label{th:id_nop_ptscore}
Suppose we have DGP \ref{ass:pt_inj} and model \eqref{model_indep} with $n=d$.
For the model, assume \ref{ass:model} 
and\renewcommand{\labelenumi}{\textbf{\emph{(M3')}}}\begin{inparaenum}
\def\theenumi{\textbf{(M3')}}
\item \label{ass:match_noise} \emph{(Noise matching) $p_{\e}=p_{\beps}$ and $\vk(\x)=k\vk'(\x),k \to 0$}.
\end{inparaenum} 
Assume further that \ref{ass:inv_jac_smp} and\renewcommand{\labelenumi}{\textbf{\emph{(D2)}}}\begin{inparaenum}
\def\theenumi{\textbf{(D2)}}
\item \label{ass:g_bcovar_smp} \emph{(Balance from data) $\mathcal{A}_0=\mathcal{A}_1$ in \eqref{eq:class}}.
\end{inparaenum}
Then, if $\vaeobsparam=\trueobs$; conclusions 1) and 2) in Proposition \ref{th:id_nop_pscore} hold with $\PS$ replaced with $\PS_t$ in \ref{ass:pt_inj}; and the domain of $\vv$ becomes
$\mathcal{P}\coloneqq\{\PS_t(\vx)|p(t,\vx) > 0\}$.
\end{theorem}
Theorem \ref{th:id_nop_ptscore} implies that, without PS, we need to know or learn the distribution of hidden noise $\beps$ to have $p_{\e}=p_{\beps}$. 
Proposition \ref{th:id_nop_pscore} and Theorem \ref{th:id_nop_ptscore} achieve recovery and identification in a complementary manner; the former starts from the prior by $\PS_0=\PS_1$ and $\vh_0=\vh_1$, 
while the latter starts from the decoder by $\mathcal{A}_0=\mathcal{A}_1$ and $p_{\e}=p_{\beps}$.
We see that $\mathcal{A}_0=\mathcal{A}_1$ acts as a kind of balance because it replaces $\PS_0=\PS_1$ (balanced PtS) in Proposition \ref{th:id_nop_pscore}. We show in Sec.~\ref{sec:proofs} a sufficient and necessary condition \ref{ass:g_bcovar} on data that ensures $\mathcal{A}_0=\mathcal{A}_1$.
Note that the singularities due to $k\to 0$ (e.g., $\blambda \to \bm 0$) cancel out in \eqref{eq:class}.
See Sec.~\ref{sec:complementarity} for more on the complementarity between the two identifications.

\section{Estimation by $\beta$-Intact-VAE}
\label{sec:estimation}

\subsection{Prior as PS, posterior as PtS, and $\beta$ as regularization strength}
\label{sec:regular}

In Sec.~\ref{sec:identification}, we see that the existence of PS (Proposition \ref{th:id_nop_pscore}) is preferable in identifying the true DGP up to an equivalent expression––while Theorem \ref{th:id_nop_ptscore} allows us to deal with PtS by adding other conditions.
In learning our model with data, we assume that there is a PtS, and the decomposition of \ref{ass:pt_inj} holds. However, such decompositions are not unique in general, and they are equivalent for CATE identification; and those with more balanced equivalent PtS are preferable. In this sense, we want to not only recover PtS, but also \textit{discover} equivalent PS, if possible. This idea is common in practice. For example, in a real-world nutrition study \citep{huang2017joint}, a reduction of 11 covariates discovers a 1-dimensional linear PS.

We consider two ways to discover and recover a equivalent PS (or balanced PtS) by a VAE. One is to use a prior which does not depend on $t$, indicating a preference for PS. Namely, we set $\blambda_0=\blambda_1=:\biglambda$ and have $p_{\biglambda}(\vz|\vx)$ as the prior in \eqref{model_indep}. The decoder and encoder are factorized Gaussians:
\begin{equation}
\label{eq:enc}
    p_{\vf,\vg}(\vy|\vz,t) =\mathcal{N}(\vy ; \vf_{t}(\vz) , \diag(\vg_{t}(\vz))),
    \mkern9mu
    q_{\bphi}(\vz|\vx,\vy,t)= \mathcal{N} ( \vz ; \vr_{t}(\vx,\vy) , \diag(\vs_{t}(\vx,\vy))),
\end{equation}
where $\bphi =(\vr,\vs)$. 
The other is to introduce a hyperparameter $\beta$ in the ELBO 
as in $\beta$-VAE \citep{higgins2016beta}. The modified ELBO with $\beta$, up to the additive constant, is derived as:
\begin{equation}
\label{eq:elbo_imp}
      \E_{\data} \{- \beta\KL(q_{\bphi}\Vert p_{\biglambda}) -\E_{\vz\sim q_{\bphi}}[(\vy-\vf_{t}(\vz))^2/2\vg^2_{t}(\vz)] - \E_{\vz\sim q_{\bphi}}\log|\vg_{t}(\vz)| \}.
\end{equation}
For convenience, here and in $\mathcal{L}_{\vf}$ in Sec.~\ref{sec:balance}, we omit the summation as if $\y$ is univariate. 
The encoder $q_{\bphi}$ depends on $t$ and can realize a PtS. With $\beta$, we control the trade-off between the first and second terms: the former is the divergence of the posterior from the balanced prior, and the latter is the reconstruction of the outcome. Note that a larger $\beta$ encourages the conditional balance $\z\independent\rt|\x$ on the posterior. By choosing $\beta$ appropriately, e.g., by validation, the ELBO can recover a balanced PtS while fitting the outcome well. In summary, we base the estimation on Proposition \ref{th:id_nop_pscore} and PS as much as possible, but step into Theorem \ref{th:id_nop_ptscore} and noise modeling required by $p_{\e}=p_{\beps}$ when necessary. 

Note also that the parameters $\vg$ and $\vk$, which model the outcome noise and express the uncertainty of the prior, respectively,  are both learned by the ELBO. This deviates from the theoretical conditions described in Sec.~\ref{sec:identification}, but it is more practical and yields better results in our experiments.
See Sec.~\ref{sec:elbo_disc} for more ideas and connections behind the ELBO. 

Once the VAE is learned\footnote{As usual, we expect the variational inference and optimization procedure to be (near) optimal; that is, consistency of VAE. \textit{Consistent estimation} using the prior is a direct corollary of the consistent VAE. see Sec.~\ref{sec:consist_vae} for formal statements and proofs. Under Gaussian models, it is possible to prove the consistency of the posterior estimation, as shown in \citet{bonhomme2019posterior}.} by the ELBO, the estimate of the expected potential outcomes is given by: 
\begin{equation}
\label{eq:cate_est}
    \hat{\mu}_{\hat{t}}(\vx)=\E_{q(\vz|\vx)}\vf_{\hat{t}}(\vz)=\E_{\data|\vx \sim p(\vy,t|\vx)}\E_{\vz \sim q_{\bphi}}\vf_{\hat{t}}(\vz),\; {\hat{t}} \in \binset,
\end{equation}
where $q(\vz|\vx)\coloneqq \E_{p(\vy,t|\vx)}\encoderparam$
is the aggregated posterior. We mainly consider the case where $\vx$ is observed in the data, and the sample of $(\y,\rt)$ is taken from the data given $\x=\vx$. When $\vx$ is not in the data, we replace $q_{\bphi}$ with $p_{\biglambda}$ in (\theequation) (see Sec.~\ref{sec:prepost} for details and \ref{sec:add_exp} for results). Note that ${\hat{t}}$ in \eqref{eq:cate_est} indicates a counterfactual assignment that may not be the same as the factual $\rt=t$ in the data. That is, we set $\rt=\hat{t}$ in the decoder. The assignment is not applied to the encoder which is learned from factual $\x,\y,\rt$ (see also the explanation of $\epsilon_{CF,t}$ in Sec.~\ref{sec:balance}). 
The overall \textbf{algorithm} steps are i) train the VAE using \eqref{eq:elbo_imp}, and ii) infer CATE $\hat{\tau}(\vx)=\hat{\mu}_1(\vx)-\hat{\mu}_0(\vx)$ by \eqref{eq:cate_est}.

\subsection{Conditionally balanced representation learning}
\label{sec:balance}

We formally justify our ELBO \eqref{eq:elbo_imp} from the BRL viewpoint. We show that the conditional BRL via the first term of the ELBO results from bounding a CATE error; particularly, the error due to the imprecise recovery of $\vj_t$ in \ref{ass:pt_inj} is controlled by the ELBO. Previous works \citep{shalit2017estimating,lu2020reconsidering} instead focus on unconditional balance and bound PEHE which is marginalized on $\x$. Sec.~\ref{sec:ihdp} experimentally shows the advantage of our bounds and ELBO. Further, we connect the bounds to identification and consider noise modeling through $\vg_t(\vz)$. Sec Sec.~\ref{sec:novel_bound} for detail. 
We introduce the objective that we bound. Using \eqref{eq:cate_est} to estimate CATE, $\hat{\tau}_{\vf}(\vz)\coloneqq\vf_1(\vz)-\vf_0(\vz)$ is marginalized on $q(\vz|\vx)$. On the other hand, the \textit{true CATE}, given the covariate $\vx$ or score $\vz$, is:
\begin{equation}
\label{eq:taus}
    \tau(\vx)=\vj_1(\PS_1(\vx))-\vj_0(\PS_0(\vx)),\quad \tau_{\vj}(\vz)=\vj_1(\vz)-\vj_0(\vz),
\end{equation}
where
$\vj_t$ is associated with a balanced PtS $\PS_t$ discovered as the target of recovery by our VAE. 
Accordingly, given $\vx$,
the \textit{error of posterior CATE}, with or without knowing $\PS_t$, is defined as
\begin{equation}
    \epsilon_{\vf}^{*}(\vx) \coloneqq \E_{q(\vz|\vx)}(\hat{\tau}_{\vf}(\vz)-\tau(\vx))^2;\quad
    \epsilon_{\vf}(\vx) \coloneqq \E_{q(\vz|\vx)}(\hat{\tau}_{\vf}(\vz)-\tau_{\vj}(\vz))^2.
\end{equation}
We bound $\epsilon_{\vf}$ instead of $\epsilon_{\vf}^{*}$ because the error between $\tau(\x)$ and $\tau_{\vj}(\z)$ is small––if the balanced $\PS_t$ is recovered, then $\vz\approx\PS_0(\vx)\approx\PS_1(\vx)$ in \eqref{eq:taus}. We consider the error between $\hat{\tau}_{\vf}$ and $\tau_{\vj}$ below. 
We define the risks of outcome regression, into which $\epsilon_{\vf}$ is decomposed.
\begin{definition}[CATE risks]
Let 
$\y({\hat t})|\PS_{\hat t}(\x) \sim p_{\y({\hat t})|\PS_{\hat t}}(\vy|P)$ and 
$q_t(\vz|\vx)\coloneqq q(\vz|\vx,t)=\E_{p(\vy|\vx,t)} q_{\bphi}$. The \textit{potential outcome loss} at $(\vz, t)$, \textit{factual risk}, and \textit{counterfactual risk} are: 
\begin{equation*}
\begin{gathered}
    \mathcal{L}_{\vf}(\vz,\hat{t})\coloneqq \E_{p_{\y({\hat t})|\PS_{\hat t}}(\vy|P=\vz)}(\vy-\vf_{\hat t}(\vz))^2/\vg_{\hat t}(\vz)^{2}
    =\vg_{\hat t}(\vz)^{-2}\textstyle\int (\vy-\vf_{\hat t}(\vz))^2 p_{\y({\hat t})|\PS_{\hat t}}(\vy|\vz)d\vy;
    \\
    \epsilon_{F,t}(\vx) \coloneqq \E_{q_t(\vz|\vx)}\mathcal{L}_{\vf}(\vz,t);\quad \epsilon_{CF,t}(\vx) \coloneqq \E_{q_{1-t}(\vz|\vx)}\mathcal{L}_{\vf}(\vz,t).
\end{gathered}
\end{equation*}
\end{definition}
With $\y(t)$ involved, $\mathcal{L}_{\vf}$ is a potential outcome loss on $\vf$, weighted by $\vg$. The factual and counterfactual counterparts, $\epsilon_{F,t}$ and $\epsilon_{CF,t}$, are defined accordingly. In $\epsilon_{F,t}$, unit $\vu=(\vx, \vy, t)$ is involved in the learning of $q_t(\vz|\vx)$, as well as in $\mathcal{L}_{\vf}(\vz,t)$ since $\y(t)=\vy$ for the unit. In $\epsilon_{CF,t}$, however, unit $\vu'=(\vx, \vy', 1-t)$ is involved in $q_{1-t}(\vz|\vx)$, but not in $\mathcal{L}_{\vf}(\vz,t)$ since $\y(t)\neq\vy'=\y(1-t)$. 

Thus, \textit{the regression error (second) term in ELBO \eqref{eq:elbo_imp} controls $\epsilon_{F,t}$ via factual data}.
On the other hand, $\epsilon_{CF,t}$ is not estimable due to the unobservable $\y(1-\rt)$, but 
is bounded by $\epsilon_{F,t}$ plus $M\mathbb{D}(\vx)$ in 
Theorem \ref{th:gen_bound} below––which, in turn, bounds $\epsilon_{\vf}$ by decomposing it to $\epsilon_{F,t}$, $\epsilon_{CF,t}$, and $\mathbb{V}_{\y}$.
\begin{theorem}[CATE error bound]
\label{th:gen_bound}
Assume $|\mathcal{L}_{\vf}(\vz,t)| \leq M$ and $|\vg_t(\vz)| \leq G$, then:
\begin{equation}
\begin{split}
\textstyle
    \epsilon_{\vf}(\vx) \leq 2[G^2(\epsilon_{F,0}(\vx) + \epsilon_{F,1}(\vx) +M\mathbb{D}(\vx)) - \mathbb{V}_{\y}(\vx)]
\end{split}
\end{equation}
where $\mathbb{D}(\vx)\coloneqq \sum_t\sqrt{\KL(q_t \Vert q_{1-t})/2}$,
and $\mathbb{V}_{\y}(\vx)\coloneqq\E_{q(\vz|\vx)}\sum_t\E_{p_{\y({t})|\PS_{t}}(\vy|\vz)}(\vy-\vj_{t}(\vz))^2$.
\end{theorem}
$\mathbb{D}(\vx)$ measures the imbalance between $q_t(\z|\vx)$ and is symmetric for $t$. 
Correspondingly, \textit{the KL term in ELBO \eqref{eq:elbo_imp} is also symmetric for $t$ and balances $q_t(\vz|\vx)$} by encouraging $\z\independent\rt|\x$ for the posterior. $\mathbb{V}_{\y}(\vx)$ reflects the intrinsic variance in the DGP and can not be controlled.

Estimating $G,M$ is nontrivial. Instead, we rely on $\beta$ in the ELBO to weight the terms in (\theequation). We do not need two hyperparameters since \textit{$G$ is implicitly controlled by the third term in ELBO \eqref{eq:elbo_imp}}, which is a norm constraint. 
$\beta$ is a trade-off between the conditional balance of learned PtS (affected by $\vf_t$), and precision/effective sample size of outcome regression––and can be seen as the probabilistic counterpart of \citet{tarr2021estimating} and \citet{kallus2018more}. 





\section{Experiments}

We compare our method with existing methods on three types of datasets. Here, we present two experiments; the remaining one on the Pokec dataset is deferred to Sec.~\ref{sec:pokec}.
As in previous works \citep{shalit2017estimating, louizos2017causal}, we report the absolute error of ATE $\epsilon_{ate}\coloneqq |\E_{\data}(y(1)-y(0)) - \E_{\data}\hat{\tau}(\vx)|$ and, as a surrogate of square CATE error $\epsilon_{cate}(\vx)=\E_{\data|\vx}[(y(1)-y(0))-\hat{\tau}(\vx)]^2$, the empirical PEHE $\epsilon_{pehe}\coloneqq \E_{\data}\epsilon_{cate}(\vx)$  \citep{hill2011bayesian}, which is the average square CATE error. 
Unless otherwise indicated, for each function $\vf,\vg,\vh,\vk,\vr,\vs$ in ELBO \eqref{eq:elbo_imp}, we use a multilayer perceptron, with $3 * 200$ hidden units, and ReLU activations. Further, $\biglambda=(\vh,\vk)$ depends only on $\rvx$. 
The Adam optimizer with initial learning rate $10^{-4}$ and batch size 100 is employed.
All experiments use early-stopping of training by evaluating the ELBO on a validation set. 
More details on hyper-parameters and settings are given in each experiment.

\subsection{Synthetic dataset}
\label{sec:exp_syn}

\begingroup

\begin{equation}
\label{art_model}
\begin{split}
\textstyle
    \rvw|\rvx \sim \mathcal{N}(\vh(\rvx), \vk(\rvx));\;
    \rt|\rvx \sim \bern(\logi(\omega l(\rvx)));\;
    \ry|\rvw,\rt \sim \mathcal{N}(f_{\rt}(\rvw), g_{\rt}(\rvw)). 
\end{split}
\end{equation}

\vspace{-.05in}
We generate synthetic datasets following \eqref{art_model}. Both $\x\sim \mathcal{N}(\bm\mu, \bsig)$ and $\rvw$ are factorized Gaussians. $\bm\mu, \bsig$ are randomly sampled. The functions $\vh,\vk,l$ are linear.
Outcome models $f_0,f_1$ are built by NNs with invertible activations. $\ry$ is univariate, $\dim(\x)=30$, and $\dim(\rvw)$ ranges from 1 to 5. $\rvw$ is a PS, but the dimensionality is not low enough to satisfy the injectivity in \ref{ass:p_inj}, when $\dim(\rvw) > 1$.
We have 5 different overlap levels controlled by $\omega$ that multiplies the logit value. See Sec.~\ref{sec:exp_syn_app} for details and more results on synthetic datasets.

\setlength{\columnsep}{8pt}
\setlength{\intextsep}{5pt}

\begin{wrapfigure}{r}{0.38\textwidth}
\vspace{-.35in}
  \begin{center}
    \includegraphics[width=0.38\textwidth]{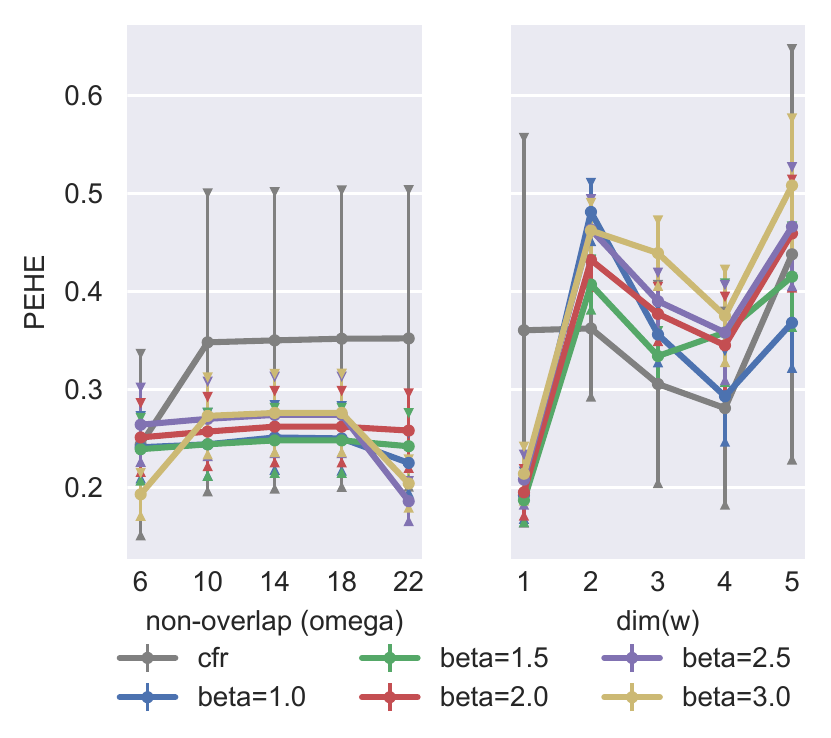}
  \end{center}
  \vspace{-.2in}
  
  \caption{\footnotesize{$\sqrt{\epsilon_{pehe}}$ on synthetic datasets. Error bar on 10 random DGPs.}}
\vspace{-.15in}
\label{fig:depn}
\end{wrapfigure}


With the same $(\dim(\rvw),\omega)$, we evaluate our method and CFR on 10 random DGPs, with different sets of functions $f,g,\vh,\vk,l$ in \eqref{art_model}. For each DGP, we sample 1500 data points, and split them into 3 equal sets for training, validation, and testing. We show our results for different hyperparameter $\beta$. For CFR, we try different balancing parameters and present the best results (see the Appendix for detail). 

In each panel of Figure \ref{fig:depn}, we adjust one of $\omega, \dim(\rvw)$, with the other fixed to the lowest. As implied by our theory, our method, with only \textit{1-dimensional} $\z$, performs much better in the left panel (where $\dim(\rvw) = 1$ satisfies $\ref{ass:p_inj}$) than in the right panel (when $\dim(\rvw) > 1$). Although CFR uses 200-dimensional representation, in the left panel our method performs much better than CFR; moreover, in the right panel CFR is not much better than ours. 
Further, our method is much more robust against different DGPs than CFR (see the error bars). Thus, the results indicate
the power of identification and recovery of scores. 
(see Figure \ref{recover} also).

\setlength{\columnsep}{8pt}
\setlength{\intextsep}{5pt}
\begin{wrapfigure}{r}{0.38\textwidth}
\vspace{-.1in}
  \begin{center}
    \includegraphics[width=0.38\textwidth]{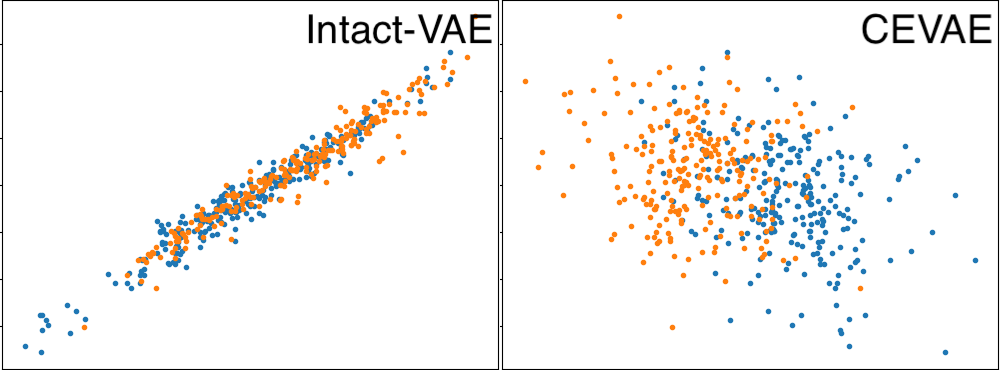}
  \end{center}
   \vspace{-.18in}
 
  \caption{\footnotesize{Plots of recovered 
  - true 
  latent. Blue: $t=0$, Orange: $t=1$.}} 
\vspace{-.1in}
\label{recover}
\end{wrapfigure}


Under the lowest overlap level ($\omega=22$), large $\beta(=2.5,3)$ shows the best results, which accords with the intuition and bounds in Sec.~\ref{sec:estimation}. 
When $\dim(\rvw) > 1$, $f_t$ in (\theequation) is non-injecitve and learning of PtS is necessary, 
and thus, larger $\beta$ has a negative effect.
In fact, $\beta=1$ is significantly better than $\beta=3$ when $\dim(\rvw)>2$. We note that our method, with a higher-dimensional $\z$, outperforms or matches CFR also under $\dim(\rvw)>1$ (see Appendix Figure \ref{fig:z200}). Thus, the performance gap under $\dim(\rvw)>1$ in Figure \ref{fig:depn} should be due to the capacity of NNs in $\beta$-Intact-VAE. In Appendix Figure \ref{fig:ate} for ATE error, CFR drops performance w.r.t overlap levels. This is evidence that CFR and its unconditional balance overly focus on PEHE (see Sec.~\ref{sec:ihdp} for detail).

When $\dim(\rvw)=1$, there are no better PSs than $\rvw$, because $f_t$ is invertible and no information can be dropped from $\rvw$.
Thus, our method stably learns $\z$ as an approximate affine transformation of the true $\rvw$, showing identification. An example is shown in Figure \ref{recover}, and more plots are in Appendix Figure \ref{fig:recover}. 
For comparison, we run CEVAE \citep{louizos2017causal}, which is also based on VAE but without identification; CEVAE shows much lower quality of recovery.
As expected, both recovery and estimation are better with the balanced prior $p_{\biglambda}(\vz|\vx)$, and we can see examples of bad recovery using $p_{\blambda}(\vz|\vx,t)$ in Appendix Figure \ref{fig:bad_recover}.

\endgroup

\subsection{IHDP benchmark dataset}
\label{sec:ihdp}


This experiment shows our conditional BRL matches state-of-the-art BRL methods and does not overly focus on PEHE. The IHDP \citep{hill2011bayesian} is a widely used benchmark dataset; while it is less known, its covariates are limited-overlapping, and thus it is used in \citet{johansson2020generalization} which considers limited overlap. The dataset is based on an RCT, but \texttt{Race} is artificially introduced as a confounder by removing all treated babies with nonwhite mothers in the data. Thus, \texttt{Race} is highly limited-overlapping, and other covariates that have high correlation to \texttt{Race}, e.g, \texttt{Birth weight} \citep{kelly2009does}, are also limited-overlapping.
See Sec.~\ref{sec:ihdp_app} for detail and more results.

There is a linear PS (linear combination of the covariates). However, most of the covariates are binary, so the support of the PS is often on small and separated intervals.
Thus, the Gaussian latent $\z$  in our model is misspecified.
We use 10-dimensional $\z$ to address this, similar to \citet{louizos2017causal}. 
We set $\beta=1$ since it works well on synthetic datasets with limited overlap.

As shown in Table \ref{t:ihdp}, $\beta$-Intact-VAE outperforms or matches the state-of-the-art methods. 
Notably, our method outperforms other generative models (CEVAE and GANITE) by large margins. 
Our method has the best ATE estimation and is only slightly worse than CFR for $\epsilon_{pehe}$. 
This fact reflects that $\epsilon_{pehe}$ is not a good criterion for CATE estimation because it is 
the marginalized CATE error––one expects less ATE error with overall less CATE error on individual-level, while PEHE focuses on those $\x=\vx$ with high probability and/or large $\epsilon_{cate}(\vx)$. Indeed, the unconditional balance
in \cite{shalit2017estimating} is based on bounding PEHE, thus results in sub-optimal ATE estimation (see also Appendix Figure \ref{fig:ate} where CFR gives larger ATE errors with less overlap). 

\begin{table}[h]

\centering
\scriptsize
\vspace{-.05in}
\caption{Errors on IHDP over 1000 random DGPs. 
``Mod.~*'' indicates the modified version with unconditional balancing hyperparameter of value ``*''. \textit{Italic} indicates where the modified version is 
significantly worse than the original.
\textbf{Bold} indicates method(s) which is 
significantly better than others. The results of other methods are taken from \cite{shalit2017estimating}, GANITE \citep{yoon2018ganite}, and CEVAE \citep{louizos2017causal}.} 

\vspace{-.05in}
\begin{tabular}{p{.6cm}p{.5cm}p{.5cm}p{.5cm}p{.5cm}p{.6cm}p{.6cm}p{.8cm}p{.8cm}p{.8cm}p{.8cm}p{.8cm}p{.8cm}}
\toprule 
Method &{TMLE} &{BNN} &{CFR} &{CF} &{CEVAE} &{GANITE} &{Ours} &{Mod. 1} &{Mod.~0.2} &{Mod.~0.1} &{Mod.~0.05} &{Mod.~0.01}\\
\midrule 
$\epsilon_{ate}$ &.30$_{\pm .01}$ &.37$_{\pm .03}$  
&.25$_{\pm.01}$  &\textbf{.18}$_{\pm.01}$ &.34$_{\pm.01}$   &.43$_{\pm.05}$  &\textbf{.177}$_{\pm .007}$ &\textit{.196}$_{\pm .008}$
&\textbf{.177}$_{\pm.007}$ &\textbf{.167}$_{\pm .005}$ &\textbf{.177}$_{\pm .006}$ &\textbf{.179}$_{\pm.006}$
\\
\midrule 
$\sqrt{\epsilon_{pehe}}$ &5.0$_{\pm .2}$ &2.2$_{\pm .1}$  
&\textbf{.71}$_{\pm.02}$ &3.8$_{\pm.2}$ &2.7$_{\pm.1 }$  &1.9$_{\pm .4}$ &.843$_{\pm .030}$ &\textit{1.979}$_{\pm .082}$ 
&\textit{1.116}$_{\pm.046}$ &$.777_{\pm .026}$ &.894$_{\pm.039}$ &.841$_{\pm .029}$
\\
\bottomrule 
\end{tabular}
\label{t:ihdp}
\vspace{-.05in}
\end{table}

To examine conditional v.s unconditional balance clearly, we modify our method and add two components for unconditional balance from \citet{shalit2017estimating} (see the Appendix), and compare the modified version to the original. 
The over-focus on PEHE of the unconditional balance from CFR is seen more clearly in the modified version. With different values of the hyperparameter,  unconditional balance does not improve (and barely affects) ATE estimation; it does affect PEHE more significantly, but often gives worse PEHE unless the hyperparameter is fine-tuned (with value 0.1). 

\vspace{-1mm}
\section{Conclusion}
\vspace{-1mm}
We proposed a method for CATE estimation under limited overlap. 
Our method exploits identifiable VAE, a recent advance in generative models, and is fully motivated and theoretically justified by causal considerations: identification, prognostic score, and balance.
Experiments show evidence that the injectivity of $\vf_t$ in our model is possibly unnecessary because $\dim(\z)>\dim(\y)$ yields better results. A theoretical study of this is an interesting future direction.
We believe that VAEs are suitable for \textit{principled} causal inference owing to their probabilistic nature, if not compromised by ad hoc heuristics (see Sec.~\ref{sec:vae_te}).
\citet[Sec.~4.2]{wu2021principled} introduce some newest ideas of this project.
 

\bibliography{main}
\bibliographystyle{iclr2022_conference}

\appendix
\clearpage

\section{Proofs}
\label{sec:proofs}
We restate our model identifiability formally.
\begin{lemma}[Model identifiability]
\label{idmodel}
Given model \eqref{model_indep} under \ref{ass:model}, for $\rt=t$, assume  

\begin{enumerate}[font=\emph]
\renewcommand{\labelenumi}{\textbf{(D1')}}
\def\theenumi{\textbf{(D1')}}


\item \label{ass:inv_jac} \emph{(Non-degenerated data for $\blambda$) there exist $2n+1$ points $\vx_0,...,\vx_{2n} \in \mathcal{X}$ such that the $2n$-square matrix $\mL_t \coloneqq [\bgamma_{t,1},...,\bgamma_{t,2n}]$ is invertible, where $\bgamma_{t,k}\coloneqq\blambda_t(\vx_k)-\blambda_t(\vx_0)$.
}
\end{enumerate}

Then, given $\rt=t$, the family is \emph{identifiable} up to an equivalence class. That is, if $p_{\vtheta}(\vy|\vx,t)=p_{\vtheta'}(\vy|\vx,t)$, we have the relation between parameters: for any $\vy_t$ in the image of $\vf_t$,
\begin{equation}
     \vf_t^{-1}(\vy_t) = \diag(\va){\vf'_t}\inv(\vy_t) + \vb 
     =:\mathcal{A}_t({\vf'_t}\inv(\vy_t))
\end{equation}
where $\diag(\va)$ is an invertible $n$-diagonal matrix and $\vb$ is a $n$-vector, both depend on $\bm\lambda_t$ and $\bm\lambda'_t$. 
\end{lemma}

Note, \ref{ass:inv_jac_smp} in the main text implies \ref{ass:inv_jac}, see Sec.~B.2.3 in \cite{khemakhem2020variational}.
The main part of our model identifiability is essentially the same as that of Theorem 1 in \cite{khemakhem2020variational}, but now adapted to include the dependency on $t$. Here we give an outline of the proof, and the details can be easily filled by referring to \cite{khemakhem2020variational}. In the proof, subscripts $t$ are omitted for convenience.

\begin{proof}[Proof of Lemma \ref{idmodel}]

Using \ref{ass:model} i) and ii) , we transform $p_{\vf,\bm\lambda}(\vy|\vx,t)=p_{\vf',\bm\lambda'}(\vy|\vx,t)$ into equality of noiseless distributions, that is, 
\begin{equation}
    q_{\vf',\bm\lambda'}(\vy)=q_{\vf,\bm\lambda}(\vy)\coloneqq p_{\bm\lambda}(\vf^{-1}(\vy)|\vx,t)vol(\mJ_{\vf^{-1}}(\vy))\mathbb{I}_{\mathcal{Y}}(\vy)
\end{equation}
where $p_{\bm\lambda}$ is the Gaussian density function of the conditional prior defined in \eqref{model_indep} and $vol(A)\coloneqq\sqrt{\det A A^T}$.
$q_{\vf',\bm\lambda'}$ is defined similarly to $q_{\vf,\bm\lambda}$.

Then, apply model \eqref{model_indep} to (\theequation), plug the $2n+1$ points from \ref{ass:inv_jac} into it, and re-arrange the resulting $2n+1$ equations in matrix form, we have
\begin{equation}
    \mathcal{F'}(\rvy)=\mathcal{F}(\rvy)\coloneqq\mL^T\vt(\vf^{-1}(\rvy))-\bbeta
\end{equation}
where $\vt(\rvz)\coloneqq(\rvz,\rvz^2)^T$ is the sufficient statistics of factorized Gaussian, and $\bbeta_t\coloneqq(\alpha_t(\vx_1)-\alpha_t(\vx_0),...,\alpha_t(\vx_{2n})-\alpha_t(\vx_0))^T$ where $\alpha_t(\rvx;\bm\lambda_t)$ is the log-partition function of the conditional prior in \eqref{model_indep}. $\mathcal{F'}$ is defined similarly to $\mathcal{F}$, but with $\vf',\bm\lambda',\alpha'$

Since $\mL$ is invertible, we have 
\begin{equation}
    \vt(\vf^{-1}(\rvy))=\mA\vt(\vf'^{-1}(\rvy))+\vc
\end{equation}
where $\mA = \mL^{-T}\mL'^{T}$ and $\vc = \mL^{-T}(\bbeta-\bbeta')$.

The final part of the proof is to show, by following the same reasoning as in Appendix B of \cite{sorrenson2019disentanglement}, that $\mA$ is a sparse matrix such that
\begin{equation}
    \mA=\begin{pmatrix}
    \diag(\va) & \mO \\
    \diag(\vu) & \diag(\va^2)
    \end{pmatrix}
\end{equation}
where $\mA$ is partitioned into four $n$-square matrices. Thus
\begin{equation}
    \vf^{-1}(\rvy)=\diag(\va)\vf'^{-1}(\rvy)+\vb
\end{equation}
where $\vb$ is the first half of $\vc$.
\end{proof}

\begin{proof}[Proof of Proposition \ref{th:id_nop_pscore}]
Under \ref{ass:p_inj}, and \ref{ass:match_pscore}, we have 
\begin{equation}
\label{eq:mean_match}
    \E_{p_{\btheta}}(\y|\x,\rt)=\E(\y|\x,\rt) \implies \vf_{t}\circ\vh(\vx)=\vj_{t}\circ\PS(\vx) \text{ on $(\vx,t)$ such that $p(t,\vx) > 0$}.
\end{equation}
We show the solution set of (\theequation) on \textit{overlapping} $\vx$ is
\begin{equation}
    \{(\vf,\vh)|\vf_t=\vj_t\circ\Delta\inv,\vh=\Delta\circ\PS,\Delta: \mathcal{P} \to \R^n \text{ is injective}\}.
\end{equation}

By \ref{ass:p_inj}\ref{ass:model}, and with injective $\vf_t,\vj_t$ and $\dim(\z)=\dim(\y) \geq \dim(\PS)$, for any $\Delta$ above, there exists a functional parameter $\vf_t$ such that $\vj_t=\vf_t\circ\Delta$. Thus, set (\theequation) is non-empty, and any element is indeed a solution because $\vf_t\circ\vh=\vj_t\circ\Delta\inv\circ\Delta\circ\PS=\vj_t\circ\PS$.

Any solution of \eqref{eq:mean_match} should be in (\theequation). A solution should satisfy $\vh(\vx)=\vf_{t}\inv\circ\vj_{t}\circ\PS(\vx)$ for both $t$ since $\vx$ is overlapping. This means the \textit{injective} function $\vf_{t}\inv\circ\vj_{t}$ should \textit{not} depend on $t$, thus it is one of the $\Delta$ in (\theequation). 

We proved conclusion 1) with $\vv\coloneqq\Delta$. And, on overlapping $\vx$, conclusion 2) is quickly seen from
\begin{equation}
    \hat{\mu}_{t}(\vx)=\vf_t(\vh(\vx))=\vj_t\circ\vv\inv(\vv\circ\PS(\vx))=\vj_{t}(\PS(\vx))=\mu_{t}(\vx).
\end{equation}

We rely on overlapping $\PS$ to work for non-overlapping $\vx$. For any $\vx_{t}$ with $p(1-t|\vx_t)=0$, to ensure $p(1-t|\PS(\vx_t))>0$, there should exist $\vx_{1-t}$ such that $\PS(\vx_{1-t})=\PS(\vx_t)$ and $p(1-t|\vx_{1-t})>0$. And we also have $\vh(\vx_{1-t})=\vh(\vx_t)$ due to \ref{ass:prepart}. Then, we have
\begin{equation}
    \hat{\mu}_{1-t}(\vx_t)=\vf_{1-t}(\vh(\vx_t))=\vf_{1-t}(\vh(\vx_{1-t}))=\vj_{1-t}(\PS(\vx_{1-t}))=\vj_{1-t}(\PS(\vx_{t}))=\mu_{1-t}(\vx_t).
\end{equation}
The third equality uses \eqref{eq:mean_match} on $(\vx_{1-t},1-t)$.
\end{proof}

Below we prove Theorem \ref{th:id_nop_ptscore} with \ref{ass:g_bcovar_smp} replaced by
\begin{enumerate}
\renewcommand{\labelenumi}{\textbf{{(D2')}}} 
\def\theenumi{\textbf{(D2')}} 
\item \label{ass:g_bcovar} \emph{(Spontaneous balance) there exist $2n+1$ points $\vx_0,...,\vx_{2n} \in \mathcal{X}$, $2n$-square matrix $\mC$, and $2n$-vector $\vd$, such that $\mL_0\inv\mL_1=\mC$ and $\bbeta_0-\mC^{-T}\bbeta_1=\vd/k$ for optimal $\blambda_t$ (see below), where $\mL_t$ is defined in \ref{ass:inv_jac}, $\bbeta_t\coloneqq(\alpha_t(\vx_1)-\alpha_t(\vx_0),...,\alpha_t(\vx_{2n})-\alpha_t(\vx_0))^T$, and $\alpha_t(\rvx;\bm\lambda_t)$ is the log-partition function of the prior in \eqref{model_indep}}.
\end{enumerate}
\ref{ass:g_bcovar} restricts the discrepancy between $\blambda_0,\blambda_1$ on $2n+1$ values of $\x$, thus is relatively easy to satisfy with high-dimensional $\x$. 
\ref{ass:g_bcovar} is general despite (or thanks to) the involved formulation. Let us see its generality even under a highly special case: 
$\mC=c\mI$ and $\vd=\bm0$. Then, $\mL_0\inv\mL_1=c\mI$ requires that, $\vh_1(\vx_k)-c\vh_0(\vx_k)$ is the same for $2n+1$ points $\vx_k$. This is easily satisfied except for $n \gg m$ where $m$ is the dimension of $\x$, which \textit{rarely} happens in practice. And, $\bbeta_0-\mC^{-T}\bbeta_1=\vd$ becomes just $\bbeta_1=c\bbeta_0$. This is equivalent to $\alpha_1(\vx_k)-c\alpha_0(\vx_k)$ same for $2n+1$ points, again fine in practice. However, the high generality comes with price. Verifying \ref{ass:g_bcovar} using data is challenging, particularly with high-dimensional covariate and latent variable. Although we believe fast algorithms for this purpose could be developed, the effort would be nontrivial. This is another motivation to use the extreme case $\blambda_0=\blambda_1$ in Sec.~\ref{sec:regular}, which corresponds to $\mC=\mI$ and $\vd=\bm0$.

\begin{proof}[Proof of Theorem \ref{th:id_nop_ptscore}]

By \ref{ass:model} and \ref{ass:pt_inj}, for any injective function $\Delta: \mathcal{P} \to \R^n$, there exists a functional parameter $\vf_t\st$ such that $\vj_t=\vf_t\st\circ\Delta$.
Let 
$\vh_t\st=\Delta\circ\PS_t$, then, clearly from \ref{ass:match_noise}, such parameters $\btheta\st=(\vf\st,\vh\st)$ are optimal: $p_{\btheta\st}(\vy|\vx,t)=\trueobs$.

Since have all assumptions for Lemma \ref{idmodel}, we have
\begin{equation}
\label{eq:id_fj}
    \Delta\circ\vj\inv(\vy)={\vf\st}\inv(\vy)=\mathcal{A}\circ{\vf}\inv(\vy)|_t, \text{ on $(\vy,t)\in \{(\vj_t\circ\PS_t(\vx),t)|p(t,\vx) > 0\}$},
\end{equation}
where $\vf$ is \textit{any} optimal parameter, and ``$|_t$'' collects all subscripts $t$. Note, except for $\Delta$, all the symbols should have subscript $t$.

Nevertheless, using \ref{ass:g_bcovar}, we can further prove $\mathcal{A}_0=\mathcal{A}_1$. 



We repeat the core quantities from Lemma \ref{idmodel} here: $\mA_t = \mL_t^{-T}\mL_t'^{T}$ and $\vc_t = \mL_t^{-T}(\bbeta_t-\bbeta'_t)$.

From \ref{ass:g_bcovar}, we immediately have
\begin{equation}
    \mL_0^{-1}\mL_1=\mL_0'^{-1}\mL'_1={\mC} \iff \mA_0=\mA_1
\end{equation}

And also,
\begin{equation}
\label{eq:mulitply2c}
    \begin{split}
        \mL_0\inv\mL_1=\mC &\iff 
        \mL_0^{-T}\mC^{-T}=\mL_1^{-T} \\
        \bbeta_0-\mC^{-T}\bbeta_1=\bbeta'_0-\mC^{-T}\bbeta'_1=\vd/k &\iff 
        \mC^{T}(\bbeta_0-\bbeta'_0)=\bbeta_1-\bbeta'_1
    \end{split}
\end{equation}
Multiply right hand sides of the two lines, we have $\vc_0=\vc_1$. Now we have $\mathcal{A}_0=\mathcal{A}_1\coloneqq\mathcal{A}$. Apply this to \eqref{eq:id_fj}, we have
\begin{equation}
    \vf_t=\vj_t\circ\vv\inv,\quad \vv\coloneqq\mathcal{A}\inv\circ\Delta
\end{equation}
for \textit{any} optimal parameters $\btheta=(\vf,\vh)$. Again, from \ref{ass:match_noise}, we have
\begin{equation}
    \vaeobsparam=\trueobs \implies p_{\beps}(\vy-\vf_{t}(\vh_{t}(\vx)))=p_{\e}(\vy-\vj_{t}(\PS_{t}(\vx)))
\end{equation}
where $p_{\beps}=p_{\e}$. And the above is only possible when $\vf_t\circ\vh_t=\vj_t\circ\PS_t$. Combined with $\vf_t=\vj_t\circ\vv\inv$, we have conclusion 1).

And conclusion 2) follows from the same reasoning as Proposition \ref{th:id_nop_pscore}, applied to both $\PS_0$ and $\PS_1$.
\end{proof}

Note, when multiplying the two lines of \eqref{eq:mulitply2c}, the effects of $k\to 0$ cancel out, and $\vc_t$ is finite and well-defined. Also, it is apparent from above proof that \ref{ass:g_bcovar} is a necessary and sufficient condition for $\mathcal{A}_0=\mathcal{A}_1$, if other conditions of Theorem \ref{idmodel} are given.

Below, we prove the results in Sec.~\ref{sec:balance}. The definitions and results work for the prior; simply \textit{replace $q_t(\vx|\vx)$ with $p_t(\vz|\vx)\coloneqq p_{\blambda}(\vz|\vx,t)$ in definitions and statements, and the proofs below hold as the same}. The dependence on $\vf$ prevail, and the superscripts are omitted. The arguments $\vx$ are sometimes also omitted. 
\begin{lemma}[Counterfactual risk bound]
\label{th:cf_bound}
Assume $|\mathcal{L}_{\vf}(\vz,t)| \leq M$, we have
\begin{equation}
    \epsilon_{CF}(\vx)\leq \textstyle\sum_t q(1-t|\vx)\epsilon_{F,t}(\vx)+M\mathbb{D}(\vx)
\end{equation}
where $\epsilon_{CF}(\vx)\coloneqq\sum_t p(1-t|\vx)\epsilon_{CF,t}(\vx)$, and $\mathbb{D}(\vx)\coloneqq \sum_t\sqrt{\KL(q_t \Vert q_{1-t})/2}$.
\end{lemma}

\begin{proof}[Proof of Lemma \ref{th:cf_bound}]
\begin{equation*}
\begin{split}
    &\epsilon_{CF} - \sum_t p(1-t|\vx)\epsilon_{F,t}\\
    &=p(0|\vx)(\epsilon_{CF,1}-\epsilon_{F,1})+p(1|\vx)(\epsilon_{CF,0}-\epsilon_{F,0})\\
    &=p(0|\vx)\int \mathcal{L}_{\vf}(\vz,1)(q_0(\vz|\vx)-q_1(\vz|\vx))d\vz+p(1|\vx)\int \mathcal{L}_{\vf}(\vz,0)(q_1(\vz|\vx)-q_0(\vz|\vx))d\vz\\
    &\leq 2M\mathbb{TV}(q_1,q_0) \leq M\mathbb{D}.
\end{split}
\end{equation*}
\end{proof}
$\mathbb{TV}(p, q)\coloneqq \frac{1}{2}\E|p(\vz)-q(\vz)|$ is the total variance distance between probability density $p,q$. The last inequality uses Pinsker's inequality $\mathbb{TV}(p, q) \leq \sqrt{\KL(p \Vert q)/2}$ twice, to get the symmetric $\mathbb{D}$.



Theorem \ref{th:gen_bound} is a direct corollary of Lemma \ref{th:cf_bound} and the following. 
\begin{lemma}
\label{th:error_deomp}
Define $\epsilon_F = \sum_t p(t|\vx)\epsilon_{F,t}$. We have
\begin{equation} 
    \epsilon_{\vf}\leq 2( G^2(\epsilon_{F}+\epsilon_{CF})-\mathbb{V}_{\y}).
\end{equation}
\end{lemma}
\textit{Simply bound $\epsilon_{CF}$ in (\theequation) by Lemma \ref{th:cf_bound}, we have Theorem \ref{th:gen_bound}.} To prove Lemma \ref{th:error_deomp}, we first examine a bias-variance decomposition of $\epsilon_F$ and $\epsilon_{CF}$.
\begin{equation}
\begin{split}
    \epsilon_{CF,t}
    &=\E_{q_{1-t}(\vz|\vx)}\vg_{t}(\vz)^{-2}\E_{p_{\y(t)|\PS_t}(\vy|\vz)}(\vy-\vf_{t}(\vz))^2\\
    &\geq  G^{-2}\E_{q_{1-t}(\vz|\vx)}\E_{p_{\y(t)|\PS_t}(\vy|\vz)}(\vy-\vf_{t}(\vz))^2\\
    &= G^{-2}\E_{q_{1-t}(\vz|\vx)}\E_{p_{\y(t)|\PS_t}(\vy|\vz)}((\vy-\vj_{t}(\vz))^2+(\vj_{t}(\vz)-\vf_{t}(\vz))^2)\\
\end{split}
\end{equation}
The second line uses $|\vg_{t}(\vz)| \leq G$, and the third line is a bias-variance decomposition. Now we can define $\mathbb{V}_{CF,t}(\vx)\coloneqq\E_{q_{1-t}(\vz|\vx)}\E_{p_{\y(t)|\PS_t}(\vy|\vz)}(\vy-\vj_{t}(\vz))^2$ and $\mathbb{B}_{CF,t}(\vx)\coloneqq\E_{q_{1-t}(\vz|\vx)}(\vj_{t}(\vz)-\vf_{t}(\vz))^2$, and we have
\begin{equation}
    \epsilon_{CF,t} \geq  G^{-2}(\mathbb{V}_{CF,t}(\vx)+\mathbb{B}_{CF,t}(\vx)) \implies \epsilon_{CF} \geq  G^{-2}(\mathbb{V}_{CF}(\vx)+\mathbb{B}_{CF}(\vx))
\end{equation}
where $\mathbb{V}_{CF}\coloneqq \sum_t p(1-t|\vx)\mathbb{V}_{CF,t}=\sum_t \E_{q(\vz,1-t|\vx)}\E_{p_{\y(t)|\PS_t}(\vy|\vz)}(\vy-\vj_{t}(\vz))^2$ and similarly $\mathbb{B}_{CF}=\sum_t\E_{q(\vz,1-t|\vx)}(\vj_{t}(\vz)-\vf_{t}(\vz))^2$. Repeat the above derivation for $\epsilon_F$, we have
\begin{equation}
    \epsilon_{F} \geq  G^{-2}(\mathbb{V}_{F}(\vx)+\mathbb{B}_{F}(\vx))
\end{equation}
where $\mathbb{V}_{F}=\sum_t \E_{q(\vz,t|\vx)}\E_{p_{\y(t)|\PS_t}(\vy|\vz)}(\vy-\vj_{t}(\vz))^2$ and $\mathbb{B}_{F}=\sum_t\E_{q(\vz,t|\vx)}(\vj_{t}(\vz)-\vf_{t}(\vz))^2$. Now, we are ready to prove Lemma \ref{th:error_deomp}.
\begin{proof}[Proof of Lemma \ref{th:error_deomp}]
\begin{equation*}
\begin{split}
    \epsilon_{\vf}&=\E_{q(\vz|\vx)}((\vf_1-\vf_0)-(\vj_1-\vj_0))^2 \\
    &=\E_q((\vf_1-\vj_1)+(\vj_0-\vf_0))^2 \\
    &\leq 2\E_q((\vf_1-\vj_1)^2+(\vj_0-\vf_0)^2) \\
    &=2\int[(\vf_1-\vj_1)^2q(\vz,1|\vx)+(\vj_0-\vf_0)^2q(\vz,0|\vx)+ \\&\qquad\quad\text{\space}(\vf_1-\vj_1)^2q(\vz,0|\vx)+(\vj_0-\vf_0)^2q(\vz,1|\vx)]d\vz \\
    &=2(\mathbb{B}_F+\mathbb{B}_{CF}) \leq 2( G^2(\epsilon_{F}+\epsilon_{CF})-\mathbb{V}_{\y})
\end{split}
\end{equation*}
\end{proof}
The first inequality uses $(a+b)^2 \leq 2(a^2+b^2)$. The next equality splits $q(\vz|\vx)$ into $q(\vz,0|\vx)$ and $q(\vz,1|\vx)$ and rearranges to get $\mathbb{B}_F$ and $\mathbb{B}_{CF}$.
The last inequality uses the two bias-variance decompositions, and $\mathbb{V}_{\y}=\mathbb{V}_{F}+\mathbb{V}_{CF}$.



\section{Additional backgrounds}
\subsection{Prognostic score and balancing score}
\label{sec:scores}
In the fundamental work of \citep{hansen2008prognostic}, prognostic score is defined equivalently to our $\PS_0$ (P0-score), but it in addition requires no effect modification to work for $\y(1)$. Thus, a useful prognostic score corresponds to our PtS. We give main properties of PtS as following. 
\begin{proposition}
\label{prop:pscore}
If $\rvv$ gives exchangeability, and $\PS_{t}(\rvv)$ is a PtS, then $\rvy(t)\independent \rvv,\rt|\PS_t$.
\end{proposition}

The following three properties of conditional independence will be used repeatedly in proofs.
\begin{proposition}[Properties of conditional independence]
\label{indep_prop}
\citep[Sec.~1.1.55]{pearl2009causality} For random variables $\rvw, \rvx, \rvy, \rvz$. We have:
\begin{equation*}
    \begin{split}
        \rvx \independent \rvy|\rvz \land \rvx \independent \rvw|\rvy, \rvz &\implies \rvx \independent \rvw,\rvy|\rvz \text{ (Contraction)}. \\
        \rvx \independent \rvw,\rvy|\rvz &\implies \rvx \independent \rvy|\rvw,\rvz \text{ (Weak union)}. \\ 
        \rvx \independent \rvw,\rvy|\rvz &\implies \rvx \independent \rvy|\rvz \text{ (Decomposition)}.
    \end{split}
\end{equation*}
\end{proposition}

\begin{proof}[Proof of Proposition \ref{prop:pscore}]
From $\rvy(t)\independent\rt|\rvv$ (\textit{exchangeability} of $\rvv$), and since $\PS_t$ is a \textit{function} of $\rvv$, we have $\rvy(t)\independent\rt|\PS_t,\rvv$ (1).

From (1) and $\rvy(t)\independent\rvv|\PS_t(\rvv)$ (definition of Pt-score), using contraction rule, we have $\rvy(t)\independent\rt,\rvv|\PS_t$ for both $t$. 
\end{proof}

Prognostic scores are closely related to the important concept of balancing score \citep{rosenbaum1983central}. Note particularly, the proposition implies $\rvy(t)\independent \rt|\PS_t$ (using decomposition rule). Thus, if $\PS(\rvv)$ is a P-score, then $\PS$ also gives weak ignorability (exchangeability and overlap), which is a nice property shared with balancing score, as we will see immediately. 

\begin{definition}[Balancing score]
\label{bscore}
$\vb(\rvv)$, a function of random variable $\rvv$, is a balancing score if $\rt \independent \rvv|\vb(\rvv)$.
\end{definition}
\begin{proposition}
Let $\vb(\rvv)$ be a function of random variable $\rvv$. $\vb(\rvv)$ is a balancing score if and only if $f(\vb(\rvv))=p(\rt=1|\rvv)\coloneqq e(\rvv)$ for some function $f$ (or more formally, $e(\rvv)$ is $\vb(\rvv)$-measurable). Assume further that $\rvv$ gives weak ignorability, then so does $\vb(\rvv)$.
\end{proposition}
Obviously, the \textit{propensity score} $e(\rvv):=p(\rvt=1|\rvv)$, the propensity of assigning the treatment given $\rvv$, is a balancing score (with $f$ be the identity function). Also, given any invertible function $\vv$, the composition $\vv \circ \vb$ is also a balancing score since $f\circ \vv^{-1}(\vv \circ \vb(\rvv))=f(\vb(\rvv))=e(\rvv)$.

Compare the definition of balancing score and prognostic score, we can say balancing score is sufficient for the treatment $\rt$ ($\rt \independent \rvv|\vb(\rvv)$), while prognostic score (Pt-score) is sufficient for the potential outcomes $\y(t)$ ($\y(t) \independent \rvv|\PS_t(\rvv)$). They complement each other; conditioning on either deconfounds the potential outcomes from treatment, with the former focuses on the treatment side, the latter on the outcomes side.

\subsection{VAE, Conditional VAE, and iVAE}
\label{vaes}
VAEs \citep{kingma2019introduction} are a class of latent variable models with latent variable $\rvz$, and observable $\rvy$ is generated by the decoder $p_\vtheta(\vy|\vz)$.
In the standard formulation \citep{DBLP:journals/corr/KingmaW13}, the variational lower bound $\mathcal{L}(\vy;\vtheta,\bm\phi)$ of the
log-likelihood is derived as:
\begin{equation}
\label{elbo_vae}
\begin{split}
      \log p(\vy) \geq \log p(\vy) - \KL(q(\vz|\vy)\Vert p(\vz|\vy)) \\ 
      = \E_{\vz \sim q}\log p_\vtheta(\vy|\vz) - \KL(q_{\bm\phi}(\vz|\vy)\Vert p(\vz)),
\end{split}
\end{equation}
where $\KL$ denotes KL divergence and the encoder $q_{\bm\phi}(\vz|\vy)$ is introduced to approximate the true posterior $p(\vz|\vy)$. The decoder $p_\vtheta$ and encoder $q_{\bm\phi}$ are usually parametrized by NNs. We will omit the parameters $\vtheta,{\bm\phi}$ in notations when appropriate.

The parameters of the VAE can be learned with stochastic gradient variational Bayes. 
With Gaussian latent variables, the KL term of $\mathcal{L}$ has closed form, while the first term can be evaluated by drawing samples from the approximate posterior $q_{\bm\phi}$ using the reparameterization trick \citep{DBLP:journals/corr/KingmaW13}, then, optimizing the evidence lower bound (ELBO) $\E_{\vy \sim \mathcal{D}}(\mathcal{L}(\vy))$ with data $\mathcal{D}$, we train the VAE efficiently. 

Conditional VAE (CVAE) \citep{sohn2015learning,kingma2014semi} adds a conditioning variable $C$, usually a class label, to standard VAE (See Figure \ref{f:vae}).
With the conditioning variable, CVAE can give better reconstruction of each class. The variational lower bound is
\begin{equation}
    \log p(\vy|\vc) \geq \E_{\vz \sim q}\log p(\vy|\vz,\vc) - \KL(q(\vz|\vy,\vc)\Vert p(\vz|\vc)). 
\end{equation}
The conditioning on $C$ in the prior is usually omitted \citep{doersch2016tutorial}, i.e., the prior becomes $\rvz \sim \mathcal{N}(\bm0, \mI)$ as in standard VAE, since the dependence between $C$ and the latent representation is also modeled in the encoder $q$. Moreover, unconditional prior in fact gives better reconstruction because it encourages learning representation independent of class, similarly to the idea of beta-VAE \citep{higgins2016beta}.

As mentioned, \textit{identifiable} VAE (iVAE) \citep{khemakhem2020variational} provides the first identifiability result for VAE, using auxiliary variable $\x$. It assumes $\rvy \independent \x|\rvz$, that is, $p(\vy|\vz,\vx)=p(\vy|\vz)$. The variational lower bound is
\begin{equation}
\begin{split}
    \log p(\vy|\vx) &\geq \log p(\vy|\vx) - \KL(q(\vz|\vy,\vx)\Vert p(\vz|\vy,\vx)) \\
    &=\E_{\vz \sim q}\log p_{\vf}(\vy|\vz) - \KL(q(\vz|\vy,\vx)\Vert p_{\bm T,\bm\lambda}(\vz|\vx)),
\end{split}
\end{equation}
where $\rvy=\vf(\rvz)+\bm\epsilon$, $\bm\epsilon$ is additive noise, and $\rvz$ has exponential family distribution with sufficient statistics $\bm T$ and parameter $\bm \lambda(\x)$. Note that, unlike CVAE, the decoder does \textit{not} depend on $\x$ due to the independence assumption.

Here, \textit{identifiability of the model} means that the functional \textit{parameters} $(\vf,\bm T,\bm\lambda)$ can be identified (learned) up to certain simple transformation. Further, in the limit of $\bm\epsilon \to \bm0$, iVAE solves the nonlinear ICA problem of recovering  $\rvz=\vf^{-1}(\rvy)$.

\section{Expositions}
The order of subsections below follows that they are referred in the main text.

\subsection{Prognostic score is more applicable than balancing score}
\label{sec:pgs_vs_bs}

Proposition 3 in \cite{d2020overlap} shows that overlapping balancing score implies overlapping $\x$, and Footnote 5 in \cite{d2020overlap} shows that overlapping $\x$ implies overlapping PS.

Here is a simple example showing overlapping PS does not imply overlapping balancing score. Let $\rt = \mathbb{I}(\x+\beps > 0)$ and $\y = \vf(|\x|, \rt) + \rve$, where $\mathbb{I}$ is the indicator function,  $\beps$ and $\rve$ are exogenous zero-mean noises, and the support of $\x$ is on the entire real line while $\beps$ is bounded. Now, $\x$ itself is a balancing score and $|\x|$ is a PS; and $|\x|$ is overlapping but $\x$ is not.

\subsection{Details and Explanations on Intact-VAE}
\label{sec:arch_details}

Generative models are useful to solve the inverse problem of recovering Pt-score. Our goal is to build a model that can be learned by VAE from observational data to obtain a PtS, or more ideally PS, via the latent variable $\z$. That is, a generative prognostic model. 

With the above goal, the generative model of our VAE is built as 
\eqref{model_indep}. Conditioning on $\x$ in the joint model $p(\vy,\vz|\vx,t)$ reflects that our estimand is CATE given $\x$. Modeling the score by a conditional distribution rather than a deterministic function is more flexible.

The ELBO of our model can be derived from standard variational lower bound as following:
\begin{equation}
\label{elbo_detail}
\begin{split}
    \log p(\vy|\vx,t) &\geq \log p(\vy|\vx,t) - \KL(q(\vz|\vx,\vy,t)\Vert p(\vz|\vx,\vy,t)) \\
    &= \E_{\vz \sim q}\log p(\vy|\vz,t) - \KL(q(\vz|\vx,\vy,t)\Vert p(\vz|\vx,t)) .
\end{split}
\end{equation}

We naturally have an identifiable conditional VAE (CVAE), as the name suggests. 
Note that \eqref{model_indep} has a similar factorization with the generative model of iVAE  \citep{khemakhem2020variational}, that is $p(\vy,\vz|\vx) = p(\vy|\vz)p(\vz|\vx)$; the first factor does not depend on $\x$.
Further, since we have the conditioning on $\rt$ in both the factors of \eqref{model_indep}, our VAE architecture is a combination of iVAE and CVAE \citep{sohn2015learning,kingma2014semi}, with $\rt$ as the conditioning variable. See Figure \ref{f:vae} for the comparison in terms of graphical models.
The core idea of iVAE is reflected in our model identifiability (see Lemma \ref{idmodel}).

\subsection{Discussions and examples of \ref{ass:p_inj}}
\label{sec:p_inj}

We focus on univariate outcome on $\R$ which is the most practical case and the intuitions apply to more general types of outcomes. Then, $\vi$, the mapping between $\mu_0$ and $\mu_1$, is monotone, i.e, either increasing or decreasing. The increasing $\vi$ means, if a change of the value of $\x$ increases (decreases)  the outcome in the treatment group, then it is also the case for the controlled group. This is often true because the treatment does \textit{not} change the mechanism how the covariates affect the outcome, under the principle of ``independence of causal mechanisms (ICM)'' \citep{janzing2010causal}. The decreasing $\vi$ corresponds to another common interpretation when ICM does not hold. Now, the treatment does change the way covariates affect $\rvy$, but in a \textit{global} manner: it acts like a ``switch'' on the mechanism: the same change of $\x$ always has \textit{opposite} effects on the two treatment groups. 

We support the above reasoning by real world examples. First we give two examples where $\mu_0$ and $\mu_1$ are both monotone increasing. This, and also that both $\mu_t$ are monotone decreasing, are natural and sufficient conditions for increasing $\vi$, though not necessary. The first example is form Health. \citep{starling2019monotone} mentions that gestational age (length of pregnancy) has a monotone increasing effect on babies' birth weight, regardless of many other covariates. Thus, if we intervene on one of the other binary covariates (say, t = receive healthcare program or not), both $\mu_t$ should be monotone increasing in gestational age. The next example is from economics. \citep{gan2016efficiency} shows that job-matching probability is monotone increasing in market size. Then, we can imagine that, with t = receive training in job finding or not, the monotonicity is not changed. Intuitively, the examples corresponds to two common scenarios: the causal effects are accumulated though time (the first example), or the link between a covariate and the outcome is direct and/or strong (the second example). 

Examples for decreasing $\vi$ are rarer and the following is a bit deliberate. This example is also about babies' birth weight as the outcome. \citep{abrevaya2015estimating} shows that, with t = mother smokes or not and $\x$ = mother's age, the CATE $\tau(\vx)$ is monotone decreasing for $20 < \vx < 26$ (smoking decreases birth weight, and the absolute causal effect is larger for older mother). On the other hand, it is shown that birth weight slightly increases (by about 100g) in the same age range in a surveyed population \citep{wang2020changing}. Thus, it is convince that, smoking changes the the tendency of birth weight w.r.t mother's age from increasing to decreasing, and gives the large decreasing of birth weight (by about 300g) as its causal effect. This could be understood: the negative effects of smoking on mother's heath and in turn on birth weight are accumulated during the many years of smoking.

\subsection{Complementarity between the two identifications}
\label{sec:complementarity}

We examine the complementarity between the two identifications more closely. The conditions \ref{ass:match_pscore} / \ref{ass:match_noise} and \ref{ass:p_inj} / \ref{ass:g_bcovar} form two pairs, and are complementary inside each pair. The first pair matches model and truth, while the second pair restricts the discrepancy between the treatment groups. 
In Theorem \ref{th:id_nop_ptscore}, \ref{ass:p_inj} ($\PS_0=\PS_1$) is replaced by \ref{ass:g_bcovar} which instead makes $\mathcal{A}_0=\mathcal{A}_1\coloneqq\mathcal{A}$ in \eqref{eq:class}. And \ref{ass:g_bcovar} is easily satisfied with high-dimensional $\x$, even if the possible values of $\mC,\vd$ are restricted to  $\mC=c\mI$ and $\vd=\bm0$ (see below). On the other hand, 
$p_{\beps}=p_{\e}$ in \ref{ass:match_noise} is impractical, but it ensures that $\vaeobsparam=\trueobs$ so that \eqref{eq:class} can be used. 
In Sec.~\ref{sec:regular}, we consider practical estimation method and introduce the \textit{regularization} that encourages learning a PtS similar to PS so that $p_{\beps}=p_{\e}$ can be relaxed.

\subsection{Ideas and connections behind the ELBO \eqref{eq:elbo_imp}}
\label{sec:elbo_disc}


\textbf{Bayesian approach is favorable\space} to express the prior belief that balanced PtSs exist and the preference for them, and to still have reasonable posterior estimation when the belief fails and learning general PtS is necessary. This is the causal importance of VAE as an estimation method for us. By the unconditional but still flexible $\biglambda$, and also the identifications, the ELBO encourages the discovery of an equivalent DGP with a balanced PtS and the recovery of it as the posterior, which still learns the dependence on $\rt$ if necessary. Moreover, $\beta$ expresses our additional knowledge (or, inductive bias) about whether or not there exist balanced PtSs (e.g., from domain expertise). 

In fact, $\beta$ connects our VAE to $\beta$-VAE \citep{higgins2016beta}, which is closely related to noise and variance control \citep[Sec.~2.4]{doersch2016tutorial}\citep{mathieu2019disentangling}.

\textbf{Considerations on noise modeling.\space} In Theorem \ref{th:id_nop_ptscore}, with large and mismatched \textit{noises} (then \ref{ass:match_noise} is easily violated), the identification of outcome model $\vf_t=\vj_t\circ\vv\inv$ would fail, and, in turn, the prior would learn confounding bias, by confusing the causal effect of $\rt$ on $\PS_{\rt}$ and the correlation between $\rt$ and $\x$. 
This is another reason to prefer $\blambda_0=\blambda_1$, besides balancing. 
On the other hand, the posterior conditioning on $\y$ provides information of noise $\e$, and it is shown in \citep{bonhomme2019posterior} that posterior effect estimation has \textit{minimum worst-case error} under model misspecification (of the noise and prior, in our case). 

Under large $\e$, a relatively small $\beta$ implicitly encourages $\vg$ \textit{smaller} than the scale of $\e$, through stressing the third term in ELBO \eqref{eq:elbo_imp}. And the the model as a whole would still learn $\trueobs$ well, because the uncertainty of $\e$ can be moved to and modeled by the prior. This is why $\vk$ is \textit{not} set to zero because learnable prior noise (variance) allows us to implicitly control $\vg$ via $\beta$. Intuitively, smaller $\vg$ strengthens the correlation between $\y$ and $\z$ in our model, and this naturally reflects that posterior conditioning on $\y$ is more important under larger $\e$. Hopefully, precise learning of outcome noise \ref{ass:match_noise} is not required, as in Proposition \ref{th:id_nop_pscore}. 

Now, it is clear that $\beta$ naturally controls at the same time noise scale and balancing. And the regularization can also be understood as an interpolation between Proposition \ref{th:id_nop_pscore} and Theorem \ref{th:id_nop_ptscore}: relying on PS, or on model identifiability; learning loosely, or precisely, the outcome regression. 
When the noise scale is different from truth, 
there would be error due to imperfect recovery of $\vj$. Sec.~\ref{sec:balance} shows that this error and balancing form a trade-off, which is adjusted by $\beta$.

\textbf{Importance of balancing from misspecification view.\space}
If we must learn an unbalanced PtS, we have larger misspecification under a balanced prior and rely more on $\y$ in the posterior. Both are bad because it is shown in \citep{bonhomme2019posterior} that posterior only helps under bounded (small) misspecification, and posterior estimator has higher variance than prior estimator (see below for an extreme case).
Again, we want a regularizer to encourage learning of PS, so that we can explore the \textit{middle ground}: relatively low-dimensional $\PS$, or relatively small $\e$.

\textbf{Example.\space} Assume the true outcome noise is (near) zero.
By setting $\beps \to \bm0$ in our model, the posterior $\vaepostparam=\vaegenparam/\vaeobsparam$ degenerates to $\vf_{\rt}\inv(\y)=\vf_{\rt}\inv(\vj_{\rt}(\PS_{\rt}))=\vv\inv(\PS_{\rt})$, a \textit{factual} PtS. However, $\vf_{1-\rt}\inv(\y)= \vf_{1-\rt}\inv(\vj_{\rt}(\PS_{\rt}))= \vv\inv(\vj_{1-\rt}\inv\circ\vj_{\rt}(\PS_{\rt})) \neq \vv\inv(\PS_{1-\rt})$, \textit{the score recovered by posterior does not work for counterfactual assignment}! The problem is, unlike $\x$, the outcome $\y=\y(\rt)$ is affected by $\rt$, and, the degenerated posterior disregards the information of $\x$ from the prior and depends exclusively on factual $(\y,\rt)$.



\subsection{Consistency of VAE and prior estimation}
\label{sec:consist_vae}

The following is a refined version of Theorem 4 in \cite{khemakhem2020variational}.
The result is proved by assuming: i) our VAE is flexible enough to ensure the ELBO is tight (equals to the true log likelihood) for some parameters; ii) the optimization algorithm can achieve the \textit{global} maximum of ELBO (again equals to the log likelihood).
\begin{proposition}[Consistency of Intact-VAE]
\label{consistency}
Given model \eqref{model_indep}\&\eqref{eq:enc}, and let $p^*(\vx,\vy,t)$ be the true observational distribution, assume 

\renewcommand{\labelenumi}{\roman{enumi})}
\begin{enumerate}
\def\theenumi{\roman{enumi})}
    \item there exists $(\bar{\vtheta}, \bar{\bm\phi})$ such that $p_{\bar{\vtheta}}(\vy|\vx,t)=p^*(\vy|\vx,t)$ and $p_{\bar{\vtheta}}(\vz|\vx,\vy,t)=q_{\bar{\bm\phi}}(\vz|\vx,\vy,t)$;
    
    \item \label{ass:prime} the ELBO $\E_{\mathcal{D} \sim p^*}(\mathcal{L}(\vx,\vy,t; \vtheta, \bm\phi))$ \eqref{elbo} can be optimized to its global maximum at $(\vtheta', \bm\phi')$;
\end{enumerate}

Then, in the limit of infinite data, $p_{\vtheta'}(\vy|\vx,t)=p^*(\vy|\vx,t)$ and $p_{\vtheta'}(\vz|\vx,\vy,t)=q_{\bm\phi'}(\vz|\vx,\vy,t)$.
\end{proposition}

\begin{proof}[Proof]
From i), we have $\mathcal{L}(\vx,\vy,t; \bar{\vtheta}, \bar{\bm\phi})=\log p^*(\vy|\vx,t)$. But we know $\mathcal{L}$ is upper-bounded by $\log p^*(\vy|\vx,t)$. So, $\E_{\mathcal{D} \sim p^*}(\log p^*(\vy|\vx,t))$ should be the global maximum of the ELBO (even if the data is finite).

Moreover, note that, for any $(\vtheta, \bm\phi)$, we have $\KL(p_{\vtheta}(\vz|\vx,\vy,t) \Vert q_{\bm\phi}(\vz|\vx,\vy,t) \geq 0$ and, in the limit of infinite data, $\E_{\mathcal{D} \sim p^*}(\log p_{\vtheta}(\vy|\vx,t)) \leq \E_{\mathcal{D} \sim p^*}(\log p^*(\vy|\vx,t))$. Thus, the global maximum of ELBO is achieved \textit{only} when $p_{\vtheta}(\vy|\vx,t)=p^*(\vy|\vx,t)$ and $p_{\vtheta}(\vz|\vx,\vy,t)=q_{\bm\phi}(\vz|\vx,\vy,t)$.
\end{proof}

Consistent prior estimation of CATE follows directly from the identifications. The following is a corollary of Theorem \ref{th:id_nop_ptscore}.
\begin{corollary}
\label{th:estimation}

Under the conditions of Theorem \ref{th:id_nop_ptscore}, 
further require the consistency of Intact-VAE.
Then, in the limit of infinite data, we have $\mu_{t}(\x)=\vf_{t}(\vh_{t}(\x))$
where $\vf,\vh$ are the optimal parameters learned by the VAE.
\end{corollary}

\subsection{Pre/Post-treatment prediction}
\label{sec:prepost}

Sampling posterior requires \textit{post-treatment} observation $(\vy, t)$. Often, it is desirable that we can also have \textit{pre-treatment} prediction for a new subject, with only the observation of its covariate $\rvx=\vx$. To this end, we use the prior as a pre-treatment predictor for $\rvz$: replace $q_{\bphi}$ with $p_{\biglambda}$ in \eqref{eq:cate_est} and get rid of the average taken on $\data$; all the others remain the same. We also have sensible pre-treatment prediction even without true low-dimensional PSs, because $p_{\biglambda}$ gives the best balanced approximation of the target PtS. The results of pre-treatment prediction are given in the experimental section \ref{sec:add_exp}. 

\subsection{Additional notes on novelties of the bounds in Sec.~\ref{sec:balance}}
\label{sec:novel_bound}

We give details and additional points regarding the novelties.
\cite{lu2020reconsidering} also use a VAE and derive bounds most related to ours. Still, our method strengthens \cite{lu2020reconsidering}, in a simpler and principled way: we distinguish true score and latent $\z$ and show that identification is the link; considering both prior and posterior, we show the symmetric nature of the balancing term and relate it to our KL term in \eqref{eq:elbo_imp}, without ad hoc regularization; moreover, we consider outcome noise modeling which is a strength of VAE and relate it to hyperparameter $\beta$. Particularly, in \citep{lu2020reconsidering}, latent variable $\z$ is confused with the true representation ($\PS_{t}$ up to invertible mapping in our case). \textit{Without} identification, the method in fact has unbounded error. Note that \cite{shalit2017estimating} do not consider connection to identification and noise modeling as well. The error between $\hat{\tau}_{\vf}$ and $\tau_{\vj}$, which we bound, is due to the unknown outcome noise that is not accounted by our Theorem \ref{th:id_nop_ptscore}; thus, the theory in Sec.~\ref{sec:balance} is complementary to that in Sec.~\ref{sec:identification}.

\section{Other related work}
\label{sec:other_related}
\subsection{Injectivity, invertibility, monotonicity, and overlap}

Let us note that \textit{any injective mapping defines an invertible mapping}, by restrict the domain of the inverse function to the range of the injective mapping. Also note that injectivity is weaker than monotonicity; a monotone mapping can be defined by an injective and \textit{order-preserving} mapping between ordered sets. Particularly, \textit{an injective and continuous mapping on $\R$ is monotone}, and many works in econometrics give examples of this case.

Many classical and recent works (with many real world applications, see C.1) in econometrics are based on monotonicity. Particularly, there is a long line of work based on \textit{monotonicity of treatment} \citep{huber2018local}. More related to our method is another line of work based on \textit{monotonicity of outcome}, see \citep{chernozhukov2013quantile} and references therein for early results. Some recent works apply monotonicity of outcome to nonparametric IV regression (NPIV) \citep{freyberger2015identification,li2017nonparametric,chetverikov2017nonparametric}, where the structural equation of the outcome is assumed to be $\ry=f(\rt)+\epsilon$, and $f$ is monotone and $\rt$ (the treatment) is often continuous. Particularly, \citep{chetverikov2017nonparametric} combines monotonicity of both treatment and outcome, and \citep{freyberger2015identification} considers \textit{discrete} treatment (note continuity or differentiability is not necessary for monotonicity). NPIV with monotone $f$ is closely related to our method, but the difference is that $\rt$ is replaced by a PtS in our method, and the PtS is recovered from observables. Finally, as we mentioned in Sec.~\ref{sec:identification}, monotonicity is a kind of shape restriction which also includes, e.g., concavity and symmetry and attracts recent interests \citep{chetverikov2018econometrics}. However, most of NPIV works focus on identifying $f$ but not directly on TEs, and we do not know any works that use monotonicity to address limited overlap.

Recently in machine learning, \citep{johansson2019support,zhang2020learning,johansson2020generalization} note the relationship between invertibility and overlap. As mentioned, \citep{johansson2020generalization} gives bounds without overlap, but the relationship between invertibility and overlap is not explicit in their theory. \citep{johansson2019support} explicitly discuss overlap and invertibility, but does not focus on TEs. \citep{zhang2020learning} assumes overlap so that identification is given, and then focuses on learning overlapping representation that preserves the overlapping the covariate. However, it does not relate invertibility and overlap, but uses invertible representation function to \textit{preserve exchangeability given the covariate}, and linear outcome regression to simply the model.  Related, our identifications required \ref{ass:prepart}, of which linearity of PtS and representation function is a sufficient condition, and our outcome model is injective, to \textit{preserve the exchangeability given the PtS}. Thus, our method works under more general setting, and arguably under weaker conditions.

\subsection{VAEs for TE estimation}
\label{sec:vae_te}

VAEs are suitable for causal estimation thanks to its probabilistic nature. However, most VAE methods for TEs, e.g. \citep{louizos2017causal,zhang2020treatment,vowels2020targeted,lu2020reconsidering}, add ad hoc heuristics into their VAEs, and thus break down probabilistic modeling, not to mention identifiable representation. Moreover, the methods rely on learning sufficient representations from \textit{proxy} variables, leading to either impractical assumptions or conceptual inconsistency, in causal identification. 

\textbf{On identification.} First, as to causal identification, \citep{louizos2017causal} assumes unobserved confounder can be recovered, which is rarely possible even under further structural assumptions  \citep{tchetgen2020introduction}, and \citep{rissanen2021critical} recently gives evidence that the method often fails. Other methods \citep{,zhang2020treatment,vowels2020targeted,lu2020reconsidering} assume unconfoundedness but still rely on proxy at least intuitively; particularly, \citep{lu2020reconsidering} factorizes the decoder as in the proxy setting. However, \textit{unconfoundedness and proxy should not be put together}. The conceptual inconsistency is that, by definition, unconfoundedness means covariates \textit{fully} control confounding, while the motivation for proxy is that unconfoundedness is often \textit{not} satisfied in practice and covariates are at best proxies of confounding, which are non-confounders causally connected to confounders \citep{tchetgen2020introduction}. 
Second, without identifiable representation, the empirical results of the methods lacks solid ground; under settings not covered by their experiments, the methods would silently fail to learn proper representations, as we show in Sec.~\ref{sec:exp_syn}.

\textbf{On ad hoc heuristics.} Ad hoc heuristics break down probabilistic modeling and / or give ELBOs that do not estimate the probabilistic models. For example, \citep{louizos2017causal} uses separated NNs for the two potential outcomes to mimic TARnet \citep{shalit2017estimating}. And, to have pre-treatment estimation, $q(\rt|\x)$ and $q(\y|\x,\rt)$ are added into the encoder. As a result, the ELBO of \citep{louizos2017causal} has two additional likelihood terms corresponding to the two distributions. \citep{zhang2020treatment} is even more ad hoc because it splits the latent variable $\z$ into three components, and applies the ad hoc tricks of \citep{louizos2017causal} to each of the component. Particularly, when constructing the encoder, \citep{zhang2020treatment} implicitly assumes the three components of $\z$ are conditional independent give $\x$, which violates the intended graphical model. 

Our method is motivated by the important concept of prognostic score, and is naturally based on \eqref{eq:cate_by_bts}.
As a consequence, our VAE architecture is a natural combination of iVAE and CVAE (see Figure \ref{f:vae}). Our ELBO \eqref{elbo} is derived by standard variational lower bound. Moreover, in our $\beta$-Intact-VAE, pre-treatment prediction is given naturally by our conditional prior, thanks to the correspondence between our model and \eqref{eq:cate_by_bts}.

\section{Details and additions of experiments}
\label{sec:add_exp}

We evaluate the post-treatment performance on training and validation set jointly (This is non-trivial. Recall the fundamental problem of causal inference). The treatment and (factual) outcome should not be observed for pre-treatment predictions, so we report them on a testing set. See also Sec.~\ref{sec:prepost} the pre/post-treatment distinction.

\subsection{Synthetic data}
\label{sec:exp_syn_app}

\begingroup

\begin{wrapfigure}{r}{0.25\textwidth}
\vspace{-.2in}
  \begin{center}
    \includegraphics[width=0.25\textwidth]{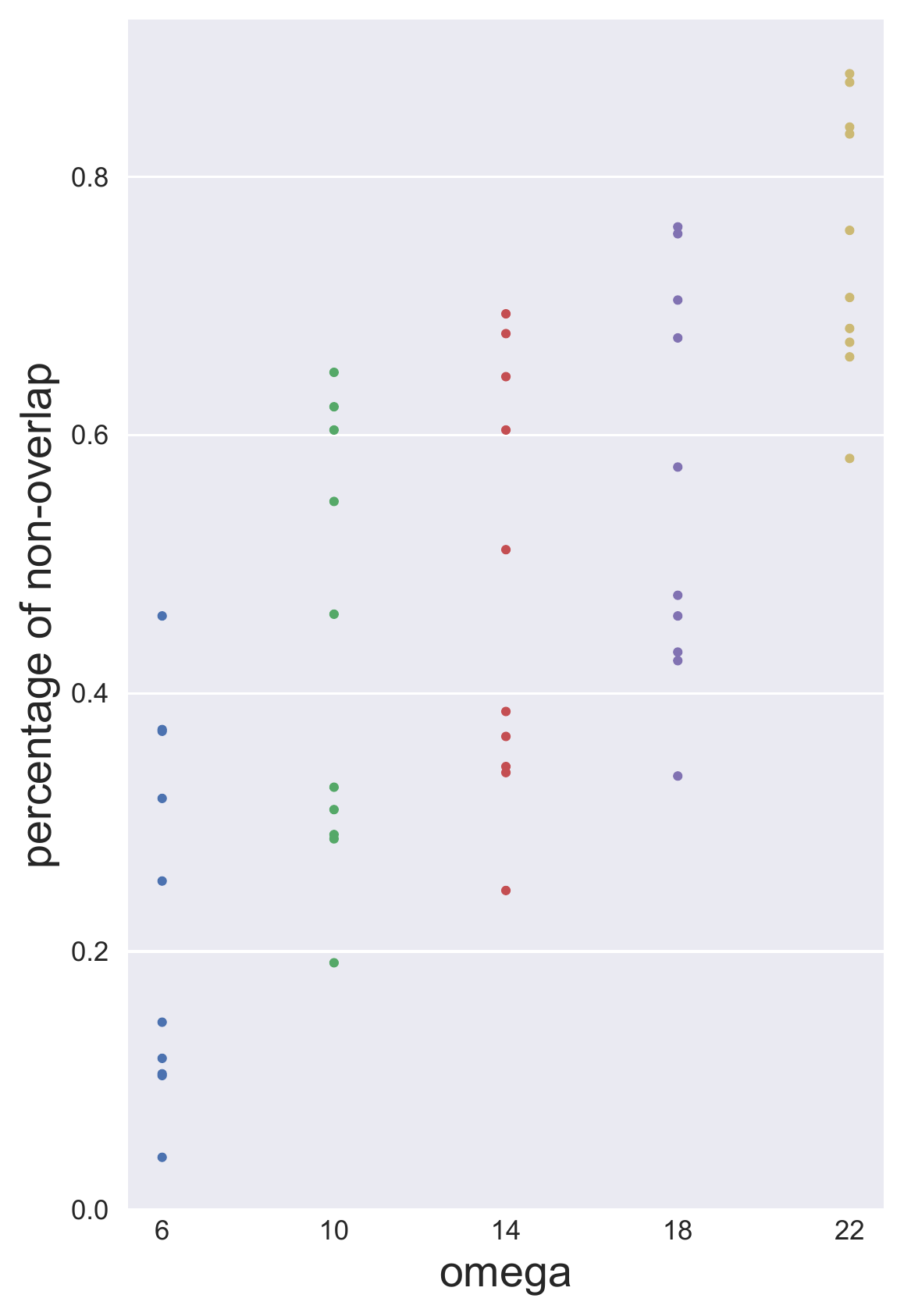}
  \end{center}
  \vspace{-.1in}
  
  \caption{\footnotesize{Degree of limited overlap w.r.t $\omega$.}}
\vspace{-.2in}
\label{fig:ol}
\end{wrapfigure}

We detail how the random parameters in the DGPs are sampled. $\mu_i$ and $\sigma_i$ are uniformly sampled in range $(-0.2, 0.2)$ and $(0, 0.2)$, respectively. The weights of linear functions $\vh,\vk,l$ are sampled from standard normal distributions. The NNs $f_0,f_1$ use leaky ReLU activation with $\alpha=0.5$ and are of 3 to 8 layers randomly, and the weights of each layer are sampled from $(-1.1, -0.9)$. To have a large but still reasonable outcome variance, the output of $f_t$ is divided by $C_t\coloneqq\Var_{\{\mathcal{D}|\rt=t\}}(f_t(\z))$. When generating DGPs with dependent noise, the variance parameter $g_t$ for the outcome is generated by adding a softplus layer after respective $f_t$, and then normalized to range $(0, 2)$.

We use the original implementation of CFR\footnote{\url{https://github.com/clinicalml/cfrnet}}. Very possibly due to bugs in implementation, the CFR version using Wasserstein distance has error of TensorFlow type mismatch on our synthetic dataset, and the CFR version using MMD diverges with very large loss value on one or two of the 10 random DGPs. We use MMD version, and, when the divergence of training happens, report the results from trained models before divergence, which still give reasonable results. We search the balancing parameter alpha in [0.16, 0.32, 0.64, 0.8, 1.28], and fix other hyperparameters as they were in the default config file.

\begin{wrapfigure}{r}{0.4\textwidth}
\vspace{-.3in}
  \begin{center}
    \includegraphics[width=0.4\textwidth]{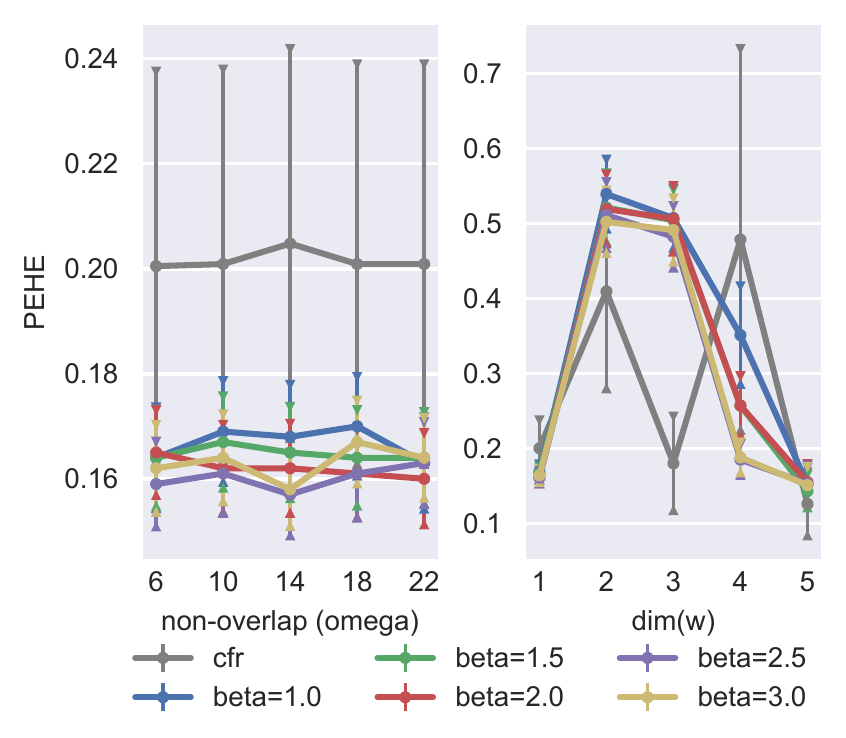}
  \end{center}
  \vspace{-.2in}
  
  \caption{\footnotesize{$\sqrt{\epsilon_{pehe}}$ on synthetic dataset, with $g_t(\rvw)=1$ in DGPs, and $\dim(\z)=200$ in our model. Error bar on 10 random DGPs.}}
\label{fig:z200}
\end{wrapfigure}

We characterize the degree of limited overlap by examining the percentage of observed values $\vx$ that give probability less than 0.001 for one of $p(t|\vx)$. The threshold is chosen so that all sample points near those values $\vx$ almost certainly belong to a single group since we have 500 sample point in total. If we regard a DGP as very limited-overlapping when the above percentage is larger than 50\%, then, as shown in Figure \ref{fig:ol}, non (all) of the 10 DGPs are very limited-overlapping  with $\omega=6$ ($\omega=22$).

\begin{wrapfigure}{r}{0.4\textwidth}
\vspace{-.5in}
  \begin{center}
    \includegraphics[width=0.4\textwidth]{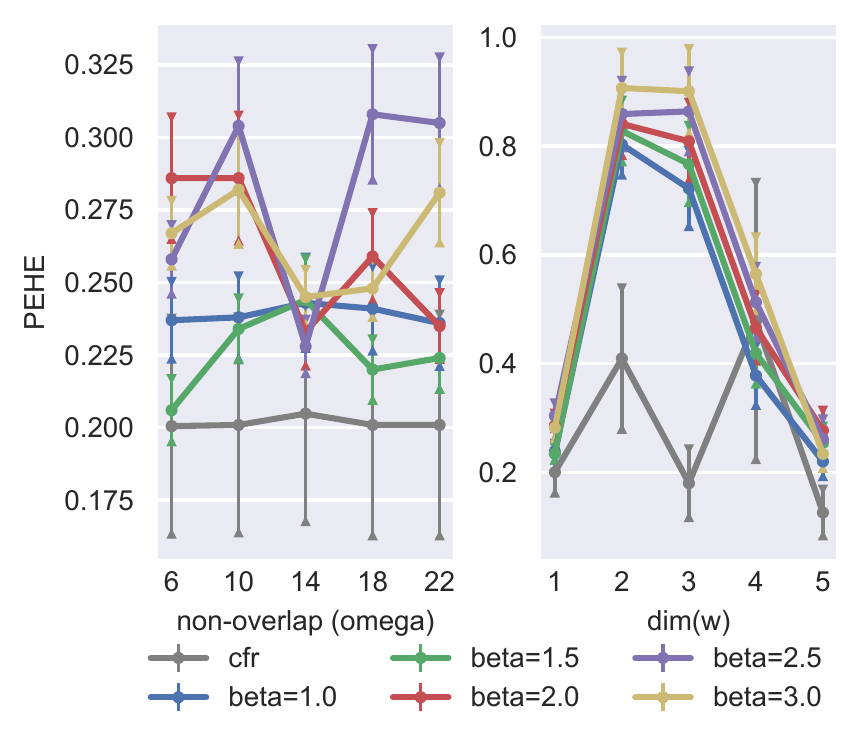}
  \end{center}
  \vspace{-.2in}
  \caption{$\sqrt{\epsilon_{pehe}}$ on synthetic dataset, with $g_t(\rvw)=1$ in DGPs. Error bar on 10 random DGPs. 
  }
\vspace{-.2in}
\label{nonl_art}
\end{wrapfigure}

For diversity of the datasets, we set $g_t(\rvw)=1$ in DGPs in Appendix. Figure \ref{fig:z200} shows, with $\dim(\z)=200$, our method works better than CFR under $\dim(\rvw)=1$ and as well as CFR under $\dim(\rvw)>1$. As mentioned in Conclusion, this indicates that the theoretical requirement of injective $\vf_t$ in our model might be relaxed. Interestingly, larger $\beta$ seems to give better results here, this is understandable because $\beta$ controls the trade-off between fitting and balancing, and the fitting capacity of our decoder is much increased with $\dim(\z)=200$. Note that the above observations on $\dim(\z)$ are not caused by fixing $g_t(\rvw)=1$ (compare Figure \ref{fig:z200} with Figure \ref{nonl_art} below).

Figure \ref{nonl_art} shows the importance of noise modeling. Compared to Figure \ref{fig:depn} in the main text, where $g_t(\rvw)$ in DGPs is not fixed, our method works worse here, particularly for large $\beta$, because now noise modeling ($\vg,\vk$ in the ELBO) only adds unnecessary complexity.
The changes of performance w.r.t different $\omega$ should be unrelated to overlap levels, but to the complexity of random DGPs; compare to Figure \ref{fig:z200}, with larger NNs in our VAE, the changes become much insignificant.
The drop of error for $\dim(\rvw) > 3$ is due to the randomness of $f$ in (\theequation). In Sec.~\ref{sec:pgs}, we saw that the 2-dimsensional PS $\PS\coloneqq(\mu_0(\x), \mu_1(\x))$ always exists under ANMs. Thus,  when $\dim(\rvw) > 2$, our method tries to recover that $\PS$, and generally performs not worse than under $\dim(\rvw)=2$, but still not better than under $\dim(\rvw)=1$.

Figure \ref{fig:ate} shows results of ATE estimation. Notably, CFR drops performance w.r.t degree of limited overlap. Our method does not show this tendency except for very large $\beta$ ($\beta=3$). This might be another evidence that CFR and its unconditional balancing overfit to PEHE (see Sec.~\ref{sec:ihdp}). Also note that, under $\dim(\rvw)=1$, $\beta=3$ gives the best results for ATE although it does not work well for PEHE, and we do not know if this generalizes to the conclusion that large $\beta$ gives better ATE estimation under the existence of PS, but leave this for future investigation.

\begin{wrapfigure}{r}{0.4\textwidth}
\vspace{-.2in}
  \begin{center}
    \includegraphics[width=0.4\textwidth]{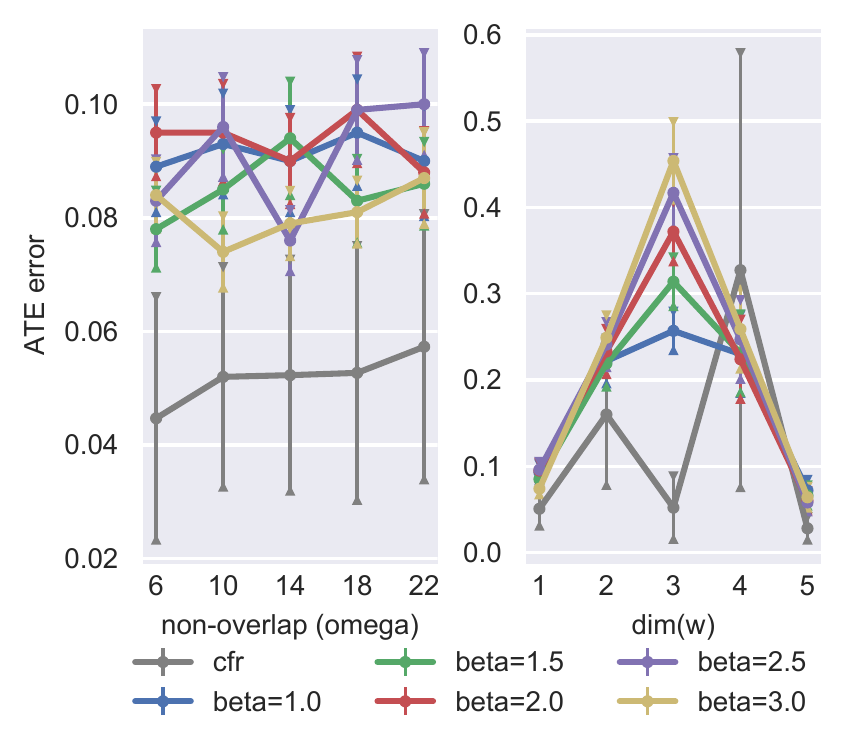}
  \end{center}
  \vspace{-.1in}
  
  \caption{\footnotesize{$\epsilon_{ate}$ on synthetic dataset, with $g_t(\rvw)=1$ in DGPs. Error bar on 10 random DGPs.}}
\vspace{-.1in}
\label{fig:ate}
\end{wrapfigure}

Figure \ref{fig:pre} shows results of pre-treatment prediction. In left panel, both our method and CFR perform only slightly worse than post-treatment. This is reasonable because here we have PS $\rvw$ with $\dim(\rvw)=1$, there is no need to learn PtS. In the right panel, we also do not see significant drop of performance compared to post-treatment. This might be due to the hardness of learning balanced PtS in this dataset, and posterior estimation does not give much improvements.

\begin{wrapfigure}{r}{0.4\textwidth}
\vspace{-.3in}
  \begin{center}
    \includegraphics[width=0.4\textwidth]{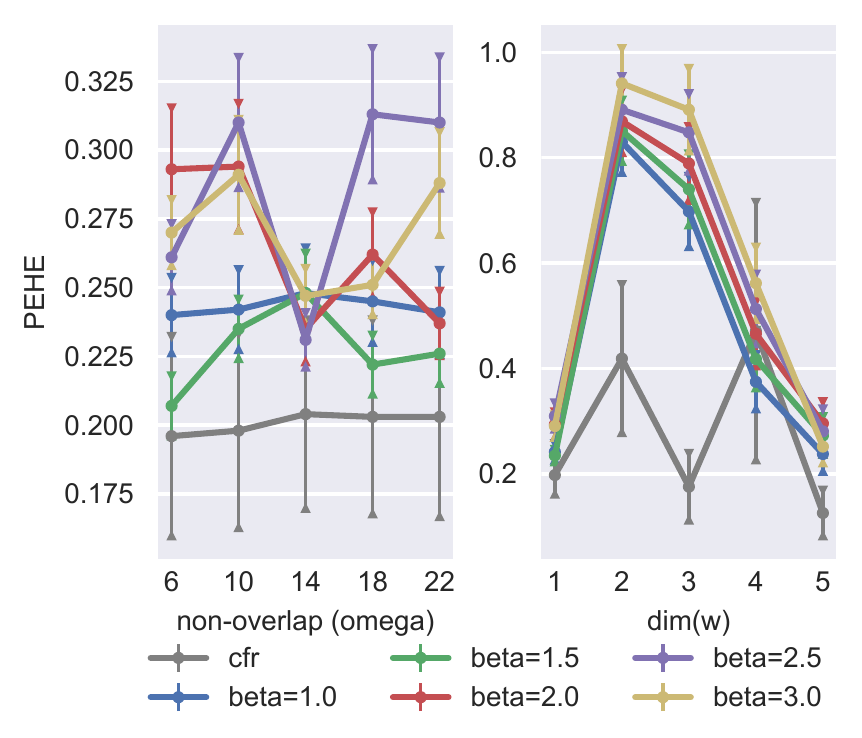}
  \end{center}
  \vspace{-.1in}
  
  \caption{\footnotesize{\textit{Pre-treatment} $\sqrt{\epsilon_{pehe}}$ on synthetic dataset. Error bar on 10 random DGPs.}}
\label{fig:pre}
\end{wrapfigure}

You can find more plots for latent recovery at the end of the paper.

\subsection{IHDP} 
\label{sec:ihdp_app}

IHDP is based on an RCT where each data point represents a baby with 25 features (6 continuous, 19 binary) about their birth and mothers. \texttt{Race} is introduced as a confounder by artificially removing all treated children with nonwhite mothers. There are 747 subjects left in the dataset. The outcome is synthesized by taking the covariates (features excluding \texttt{Race}) as input, hence \textit{unconfoundedness} holds given the covariates. Following previous work, we split the dataset by 63:27:10 for training, validation, and testing. Note, there is no ethical concerns here, because the treatment assignment mechanism is artificial by processing the data. Also our results are only quantitative and we make no ethical conclusions.

The generating process is as following \citep[Sec.~4.1]{hill2011bayesian}.
\begin{equation}
    \ry(0) \sim \mathcal{N}(e^{\va^T(\x+\vb)},1),\quad \ry(1) \sim \mathcal{N}(\va^T\x-c,1),
\end{equation}
where $\va$ is a random coefficient, $\vb$ is a constant bias with all elements equal to $0.5$, and $c$ is a random parameter adjusting degree of overlapping between the treatment groups. As we can see, $\va^T\x$ is a true PS. 
As mentioned in the main text, the PS might be discrete. Thus, this experiment also shows the importance of VAE, even if an apparent PS exists. Under \textit{discrete} PSs, training an regression based on Proposition \ref{th:id_nop_pscore} is hard, but our VAE works well.

The two added components in the modified version of our method are as following. First, we build the two outcome functions $\vf_t(\rvz),t=0,1$ in our learning model \eqref{model_indep}, using two separate NNs. Second, we add to our ELBO \eqref{elbo} a regularization term, which is the Wasserstein distance \citep{cuturi2013sinkhorn} between $\E_{\data\sim p(\x|\rt=t)}p_{\biglambda}(\rvz|\rvx),t\in\binset$. As shown in Table \ref{t:ihdp_mod}, best unconditional balancing parameter is 0.1. Larger parameters gives much worse PEHE and does not improve ATE estimation. Smaller parameters are more reasonable but still do not improve the results. The overall tendency is clear. Compared to ours, CFR with its unconditional balancing does not improve ATE estimation, it may improve PEHE results with fine tuned parameter, but possibly at the price of worse ATE estimation.

\endgroup


\begin{table}[h]

\centering
\scriptsize
\vspace{-.1in}
\caption{Performance of modified version with different unconditional balancing parameter, the values of which are shown after ``Mod.''.} 

\begin{tabular}{p{1.cm}p{.8cm}p{.8cm}p{.8cm}p{.8cm}p{.8cm}p{.8cm}p{.8cm}p{.8cm}}
\toprule 
Method &{Ours} &{Mod. 1} &{Mod. 0.2} &{Mod. 0.1} &{Mod. 0.05} &{Mod. 0.01} &CFR \\
\midrule 
$\epsilon_{ate}$ &.177$_{\pm .007}$ &.196$_{\pm .008}$
&.177$_{\pm.007}$ &.167$_{\pm .005}$ &.177$_{\pm .006}$ &.179$_{\pm.006}$ &.25$_{\pm.01}$
\\
\midrule 

$\sqrt{\epsilon_{pehe}}$ &.843$_{\pm .030}$ &1.979$_{\pm .082}$ 
&1.116$_{\pm.046}$ &$.777_{\pm .026}$ &.894$_{\pm.039}$ &.841$_{\pm .029}$ &.71$_{\pm.02}$
\\

\bottomrule 
\end{tabular}

\label{t:ihdp_mod}
\end{table}

Table \ref{t:ihdp_pre} shows pre-treatment results, All methods gives reasonable results.
\begin{table}[h]

\centering
\scriptsize
\vspace{-.1in}
\caption{\textit{Pre-treatment} Errors on IHDP over 1000 random DGPs. We report results with $\dim(\z)=10$. 
\textbf{Bold} indicates method(s) which is \textit{significantly} better. The results are taken from \cite{shalit2017estimating}, except GANITE \citep{yoon2018ganite} and CEVAE \citep{louizos2017causal}.} 

\begin{tabular}{p{1.cm}p{.5cm}p{.5cm}p{.5cm}p{.5cm}p{.7cm}p{.7cm}p{.8cm}p{1.2cm}}
\toprule 
Method &{TMLE} &{BNN} &{CFR} &{CF} &{CEVAE} &{GANITE} &{Ours} \\
\midrule 
pre-$\epsilon_{ate}$ &NA &.42$_{\pm .03}$
&.27$_{\pm.01}$ &.40$_{\pm.03}$ &.46$_{\pm.02}$ &.49$_{\pm.05}$  &\textbf{.211}$_{\pm .011}$
\\
\midrule 

pre-$\sqrt{\epsilon_{pehe}}$ &NA &2.1$_{\pm .1}$ 
&\textbf{.76}$_{\pm.02}$ &3.8$_{\pm.2}$ &2.6$_{\pm.1}$ &2.4$_{\pm .4}$ &.946$_{\pm .048}$
\\

\bottomrule 
\end{tabular}
\label{t:ihdp_pre}
\vspace{-.1in}
\end{table}

\subsection{Pokec Social Network Dataset}
\label{sec:pokec}

This experiment shows our method is the best compared with the methods specialized for networked deconfounding, a challenging problem in its own right. Thus, our method has the potential to work under \textit{unobserved confounding}, but we leave detailed experimental and theoretical investigation to future.

Pokec \citep{leskovec2014snap} is a real world social network dataset. We experiment on a semi-synthetic dataset based on Pokec, which was introduced in \citep{veitch2019using}, and use exactly the same pre-processing and generating procedure. The pre-processed network has about 79,000 vertexes (users) connected by 1.3 $\times 10^6$ undirected edges.  The subset of users used here are restricted to three living districts which are within the same region. The network structure is expressed by binary adjacency matrix $\mG$. Following \citep{veitch2019using}, we split the users into 10 folds, test on each fold and report the mean and std of pre-treatment ATE predictions. We further separate the rest of users (in the other 9 folds) by $6:3$, for training and validation. 

Each user has 12 attributes, among which  \texttt{district}, \texttt{age}, or \texttt{join date} is used as a confounder $\ru$ to build 3 different datasets, with remaining 11 attributes used as covariate $\rvx$. Treatment $\rt$ and outcome $\rvy$ are synthesised as following:
\begin{equation}
\label{pokec}
    \rt \sim \bern(g(\ru)), \quad \rvy = \rt + 10(g(\ru)-0.5) + \epsilon,
\end{equation}
where $\epsilon$ is standard normal. Note that \texttt{district} is of 3 categories; \texttt{age} and \texttt{join date} are also discretized into three bins. $g(\ru)$, which is a PS, maps these three categories and values to $\{0.15, 0.5, 0.85\}$. 

$\beta$-Intact-VAE is expected to learn a PS from $\mG, \rvx$, if we can exploit the network structure effectively. 
Given the huge network structure, most users can practically be identified by their attributes and neighborhood structure, which means $\ru$ can be roughly seen as a deterministic function of $\mG, \rvx$. 
This idea is comparable to Assumptions 2 and 4 in \citep{veitch2019using}, which postulate directly that a balancing score can be learned in the limit of infinite large network. 
To extract information from the network structure, we use Graph Convolutional Network (GCN) \citep{DBLP:conf/iclr/KipfW17} in conditional prior and encoder of $\beta$-Intact-VAE. The implementation details are given at the end of this subsection.

Table \ref{tbl:Pokec} shows the results. 
The pre-treatment $\sqrt{\epsilon_{pehe}}$ for \texttt{Age}, \texttt{District}, and \texttt{Join date} confounders are 1.085, 0.686, and 0.699 respectively, practically the same as the ATE errors. Note that, \cite{veitch2019using} does not give individual-level prediction.

\begin{table}[h]
\centering
\scriptsize

\caption{Pre-treatment ATE on Pokec. Ground truth ATE is 1, as we can see in \eqref{pokec}. ``Unadjusted'' estimates ATE by $\E_{\mathcal{D}}(y_1)-\E_{\mathcal{D}}(y_0)$. ``Parametric'' is a stochastic block model for networked data \citep{gopalan2013efficient}. ``Embed-'' denotes the best alternatives given by \citep{veitch2019using}. \textbf{Bold} indicates method(s) which is \textit{significantly} better than all the others. We report results with 20-dimensional latent $\z$. The results of the other methods are taken from \citep{veitch2019using}.} \label{tbl:Pokec}

\begin{tabular}{l|c|c|c}
\toprule 
              & \texttt{Age}   & \texttt{District} & \texttt{Join Date} \\
              \midrule
Unadjusted    & 4.34 $\pm$ 0.05  & 4.51 $\pm$ 0.05  & 4.03 $\pm$ 0.06   \\
Parametric  & 4.06 $\pm$ 0.01   & 3.22 $\pm$ 0.01  & 3.73 $\pm$ 0.01   \\
Embedding-Reg.  & 2.77 $\pm$ 0.35  & \textbf{1.75} $\pm$ 0.20   & 2.41 $\pm$ 0.45   \\
Embedding-IPW  & 3.12 $\pm$ 0.06  & \textbf{1.66} $\pm$ 0.07   & 3.10 $\pm$ 0.07   \\
Ours           &\textbf{2.08} $\pm$ 0.32   & \textbf{1.68} $\pm$ 0.10   & \textbf{1.70} $\pm$ 0.13  \\
\bottomrule
\end{tabular}
\end{table}


To extract information from the network structure, we use Graph Convolutional Network (GCN) \citep{DBLP:conf/iclr/KipfW17} in conditional prior and encoder of $\beta$-Intact-VAE. A difficulty is that, the network $\mG$ and covariates $\mX$ of \textit{all} users are always needed by GCN, regardless of whether it is in training, validation, or testing phase. However, the separation can still make sense if we take care that the treatment and outcome are used only in the respective phase, e.g., $(y_m,t_m)$ of a testing user $m$ is only used in testing.

GCN takes the network matrix $\mG$ and the \textit{whole} covariates matrix $\mX \coloneqq (\vx_1^T,\dotsc,\vx_M^T)^T$, where $M$ is user number, and outputs a representation matrix $\mR$, again for all users. During training, we \textit{select} the rows in $\mR$ that correspond to users in training set. Then, treat this \textit{training representation matrix} as if it is the covariates matrix for a non-networked dataset, that is, the downstream networks in conditional prior and encoder are the same as in the other two experiments, but take $(\mR_{m,:})^T$ where $\vx_m$ was expected as input. And we have respective selection operations for validation and testing. We can still train $\beta$-Intact-VAE including GCN by Adam, simply setting the gradients of non-seleted rows of $\mR$ to 0. 

Note that GCN cannot be trained using mini-batch, instead, we perform batch gradient decent using full dataset for each iteration, with initial learning rate $10^{-2}$. We use dropout \citep{srivastava2014dropout} with rate 0.1 to prevent overfitting.






\subsection{Additional plots on synthetic datasets}
See next pages.

\clearpage


\begin{figure}[h] 
    \vspace*{-0.1in}
    \centering
    \includegraphics[width=.9\textwidth]{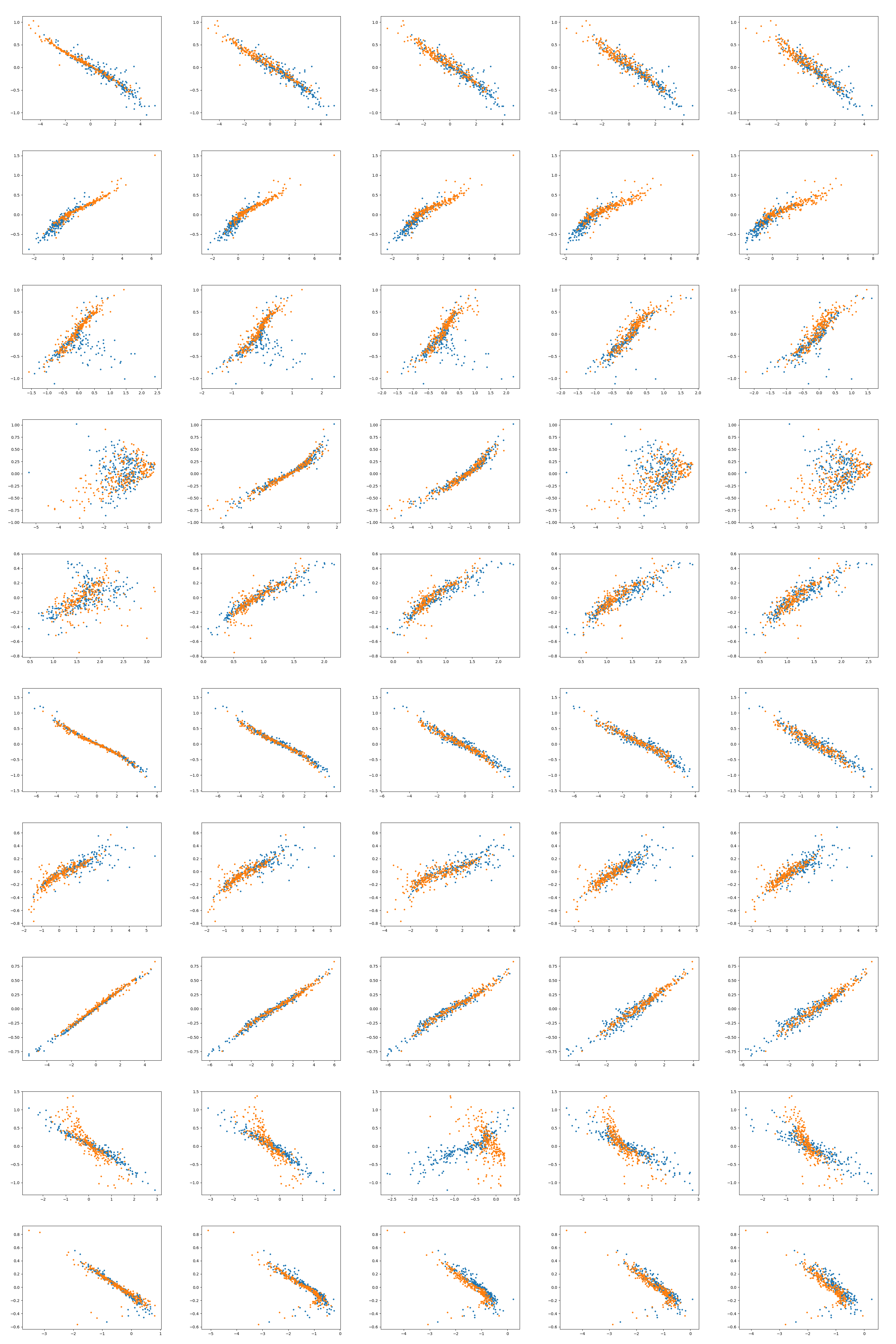}
    \vspace*{-0.1in}
    \caption{Plots of recovered-true latent. Rows: first 10 nonlinear random models, columns: outcome noise level.}
    \vspace*{-0.1in}
    \label{fig:recover}
\end{figure}

\begin{figure}[h] 
    \vspace*{-0.1in}
    \centering
    \includegraphics[width=.9\textwidth]{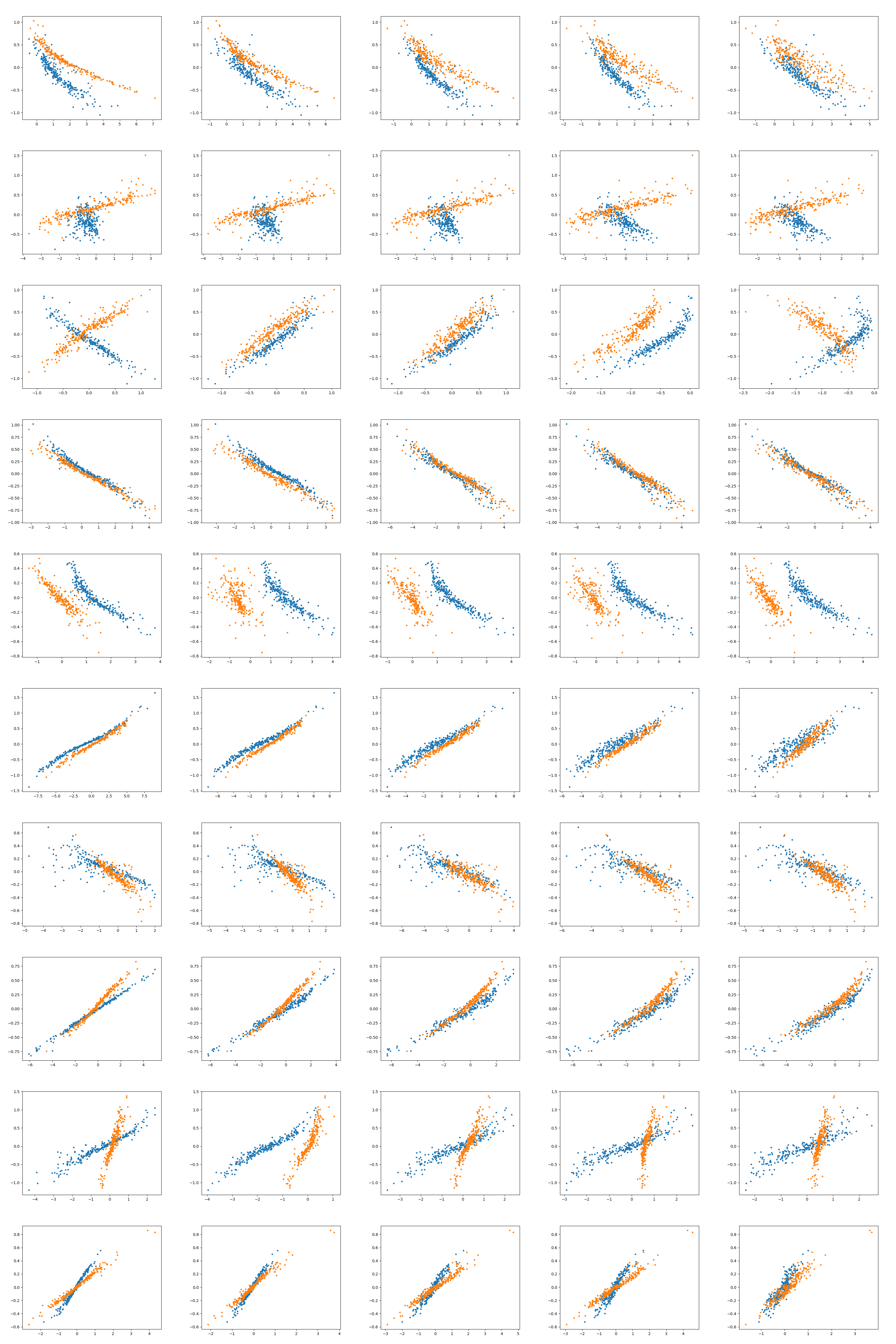}
    \vspace*{-0.1in}
    \caption{Plots of recovered-true latent. Conditional prior \textit{depends} on $t$. Rows: first 10 nonlinear random models, columns: outcome noise level. Compare to the previous figure, we can see the transformations for $t=0,1$ are \textit{not} the same, confirming the importance of balanced prior.}
    \vspace*{-0.1in}
    \label{fig:bad_recover}
\end{figure}

\clearpage

\end{document}